\documentclass[conference]{IEEEtran}
\usepackage{times}

\usepackage[numbers]{natbib}
\usepackage{multicol}
\usepackage[bookmarks=true]{hyperref}

\pdfoutput=1
\usepackage[utf8]{inputenc}
\usepackage{amsmath}
\usepackage{amssymb}
\usepackage{graphicx}
\usepackage{cancel}
\usepackage{dsfont}
\usepackage{algorithm}
\usepackage{float}
\usepackage{tikz}
\usepackage{bbm}
\usepackage{dsfont}
\usepackage{graphicx}
\usepackage{amsthm}
\usepackage{algorithmic}
\usepackage{dsfont}
\usepackage{graphicx}
\usepackage{subcaption}
\usepackage{float}
\usepackage{placeins}

\usepackage{xspace}
\usepackage{romannum}
\usepackage{xcolor,colortbl}
\usepackage{tabularx}
\usepackage[labelfont=bf]{caption}

\newcommand{\his}{\ensuremath{{\mathcal H}}\xspace}

\theoremstyle{plain}
\newtheorem{thm}{\protect\theoremname}
\theoremstyle{remark}

\theoremstyle{plain}

\theoremstyle{definition}

\providecommand{\claimname}{Claim}
\providecommand{\definitionname}{Definition}
\providecommand{\lemmaname}{Lemma}
\providecommand{\theoremname}{Theorem}

\newcommand{\bydef}{\triangleq}
\newcommand{\probd}{\mathbb{P}}
\newcommand{\prob}{P} 
\newcommand{\rew}{\rho}
\newcommand{\rewsimpl}{\Breve{\rew}}
\newcommand{\ret}{g_{k}}
\newcommand{\sep}{ | }
\newcommand{\retsimpl}{\Breve{g}_{k}}
\newcommand{\bsimpl}{\Breve{b}}

\newcommand{\lb}{l}
\newcommand{\ub}{u}

\newcommand{\mot}[3]{\probd_T({#3} \sep {#2}, {#1})}
\newcommand{\observ}[2]{\probd_Z({#2} \sep {#1})}

\newcommand{\ploss}{\texttt{PLoss}\xspace}
\newcommand{\pbloss}{\texttt{PbLoss}\xspace}
\newcommand{\extendedpomdp}{$\probd\rho$-POMDP\xspace}

\begin{document}

\title{Probabilistic Loss and its Online Characterization for Simplified Decision Making Under Uncertainty}


\author{Andrey Zhitnikov\IEEEauthorrefmark{1} and Vadim Indelman\IEEEauthorrefmark{2} \\
	\IEEEauthorrefmark{1}Technion Autonomous Systems Program \quad \IEEEauthorrefmark{2}Department of Aerospace Engineering\\
	Technion - Israel Institute of Technology, Haifa 32000, Israel\\
	\small{\texttt{andreyz@campus.technion.ac.il, vadim.indelman@technion.ac.il}}
}



%

\maketitle

\begin{abstract}
It is a long-standing objective to ease the computation burden incurred by the decision making process. Identification of this mechanism’s sensitivity to simplification has tremendous ramifications. Yet, algorithms for decision making under uncertainty usually lean on approximations or heuristics without quantifying their effect. Therefore, challenging scenarios could severely impair the performance of such methods. In this paper, we extend the decision making mechanism to the whole by removing standard 
approximations and considering all  previously suppressed stochastic sources of variability. On top of this extension,
our key contribution is a  novel framework to simplify decision making while assessing and controlling online the simplification’s impact. Furthermore, we present novel stochastic bounds on the return and characterize online  
the effect of simplification using this framework on a particular simplification technique - reducing the number of
samples in belief representation for planning. Finally, we verify the advantages of our approach through extensive simulations.
\end{abstract}

\IEEEpeerreviewmaketitle

\section{Introduction}

	

Autonomous online decision making is a fundamental aspect of intelligence.  In a partially observable setting, which is common in real world scenarios, there is no direct access to the state. Instead, the robot has to maintain a belief over the state, and reason about its evolution while accounting for different sources of uncertainty within the decision making stage. The renowned framework to do so is the Partially Observable Markov Decision Process (POMDP) \cite{Kaelbling98ai}.  A crucial element defining the robot's behavior is the reward operator.  

Solving a POMDP, i.e.,~calculating the ``right decision" in terms of an optimal action sequence or policy, involves anticipating every imaginable turn of future events and computing the \emph{returns}  based on the corresponding rewards. One common example of the return is the future cumulative reward.
The abundance of such possibilities, represented by a  belief tree,  induces a \emph{distribution over return} for every possible action plan. Therefore, choosing an optimal action (policy) is exceptionally computationally demanding \cite{papadimitriou1987complexity}.   Since a direct comparison of the distributions is not possible, a decision-maker must project the distribution of return, per a possible reactive future action sequence (policy) \cite{kochenderfer2015decision},  on some comparable space. Examples for such a projection are the expectation operator and risk-aware measures \cite{Defourny08ws},\cite{bouakiz1995target},\cite{wu1999minimizing}, such as conditional-value-at-risk (CVaR) \cite{Rockafellar00jr}.

There is a large body of algorithms to approximate decision making under uncertainty.  Classical offline methods are based on $\alpha$-vectors \cite{kurniawati2008sarsop} or value iteration \cite{bai2014integrated}.   
More recently, online methods became successful. These include,  for example, POMCP \cite{silver2010monte} and its various extensions (e.g.,~\cite{sunberg2017online}), an algorithm designed for large POMDP and based on Monte Carlo tree search. Another popular algorithm, DESPOT \cite{somani2013despot} \cite{Ye17jair}, focuses on the set of randomly sampled scenarios over the belief tree, avoiding drawbacks of  the UCT \cite{kocsis2006bandit} algorithm used in POMCP.  Presently, it has been further improved to tackle high dimensional spaces \cite{garg2019despot}. \citet{kurniawati2016online} present an interesting adaptive algorithm for dynamic environments. 
\begin{figure}[t]
	\centering 
	\includegraphics[width=\columnwidth]{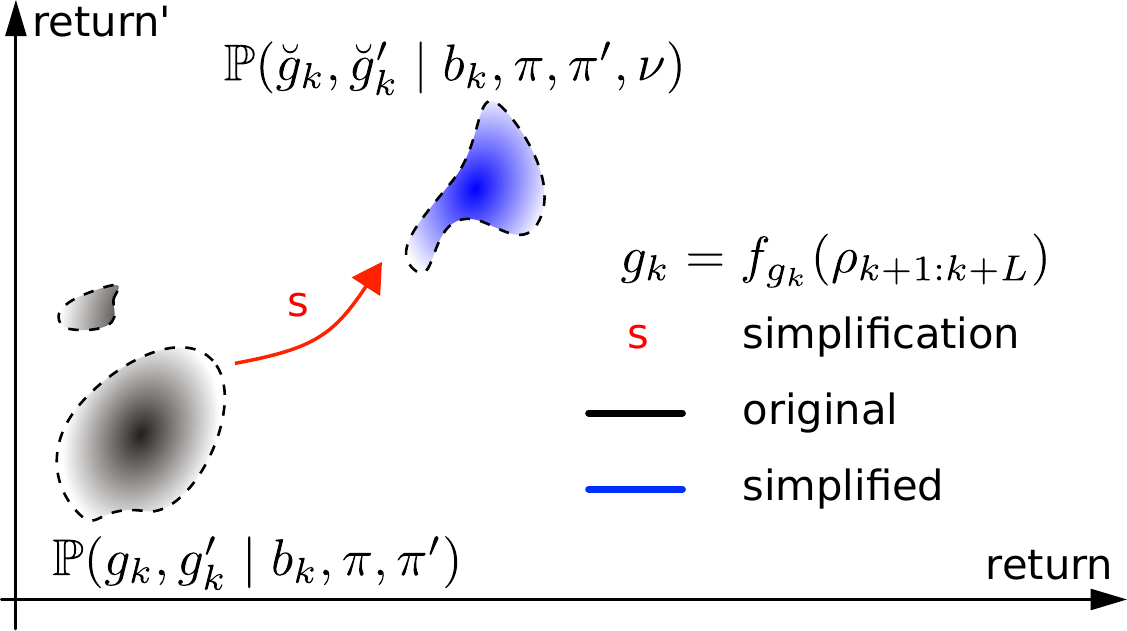}
	\caption{This figure shows alteration of the distribution of joint returns $g_k$ and $g'_k$ of two candidate policies $\pi$ and $\pi'$ as a result of simplification. Color intensity denotes distribution values. This is a conceptual illustration, i.e., we do not imply higher/lower rewards or change of support due to simplification.} \label{fig:JointBeliefRewards}
\end{figure}

Standard POMDP formulations consider state-dependent rewards. POMDP with \emph{belief-dependent rewards} received much less attention, although these rewards are essential in numerous problems, such as information gathering, autonomous navigation, and active sensing. Information theoretic rewards are especially significant for belief space planning (BSP) \cite{indelman2015planning}, \cite{fischer2020information}. \citet{araya2010pomdp} introduced $\rho$-POMDP and extended the exact $\alpha$-vectors method and a family of point based approximation algorithms to considering convex belief-dependent reward functions. Later \cite{fehr2018rho} extended their work further to  Lipschitz reward functions.  \citet{spaan2015decision} proposed to augment action space with information-reward actions.  \citet{dressel2017efficient} proposed an extension of SARSOP \cite{kurniawati2008sarsop} to specific forms of belief-dependent rewards.  

Previous techniques necessitate specific forms of reward operators. One standing-out approach \cite{farhi2019ix}, \cite{farhi2021ixBSP} presents incremental reuse of calculations between different planning sessions. Notably, that approach is formulated for general belief-dependent rewards.

We take a different path, which is to simplify the original decision making problem. In other words, instead of approximating the problem, we substitute it with a simpler one. If the order of policies with respect to the original and simplified problems' objective is preserved,  such substitution does not affect the decision making quality. Moreover, if, utilizing the simplified problem,  we can find online bounds over the returns or objective function of the original problem, it is possible to account for the simplification loss. Replacement of various parts of the decision making problem to ease the computation burden while preserving the precedence of objectives for potential action plans recently appeared in the literature under the name action consistency \cite{elimelech2020fast},\cite{elimelech2019efficient},\cite{kitanov2019topological}. Yet, these works considered a limited setting of a specific projection operator, Gaussian distributions, and maximum likelihood observations. Importantly, they formulated action consistency in an absolute way. However, an absolute action consistent simplification is not easy to achieve in complex general scenarios: to find such action consistent simplification, one has to preserve action trends for all conceivable realizations of the future for every potential sequence of actions, even if considering a specific projection operator (expectation). Unfortunately, this is exceptionally challenging without the assumption of maximum likely observation.

In this paper, we introduce a more general framework   that allows to reason leniently about simplifications that mostly, or partially, preserve action ordering, i.e.,~for some of the possible future return realizations of different actions. Given a user-provided threshold on the loss incurred by simplification, one can assess the probability of suffering loss larger than the threshold and thus provide performance guarantees. 
    
We focus  on the distribution of the returns.  This distribution conveys all the information about the decision making problem.  Our goal is to examine how the simplification method influences the performance of the decision-maker. The simplification impacts the joint distribution of the returns, as illustrated in Fig.~ \ref{fig:JointBeliefRewards}. A simplification technique could affect the dependency between returns marginals \cite{indelman2016no}, as well as the marginals themselves \cite{elimelech2019efficient}.

Ultimately, we shall perform our analysis in an online setting, meaning, without accessing the original problem. In other words, we are permitted to use only the ingredients of the simplified problem to quantify the simplification effect. Therefore, we intend to quantify distribution over loss online.

Furthermore, our study's center is belief-dependent general rewards.  In the most general formulation, reward calculation and belief update are stochastic methods.   Typically the sources of the stochasticity are sample approximations. 
To our knowledge, all existing works adopting the POMDP setting do not consider these stochastic aspects. Conventionally the reward approximation and the belief update are considered deterministic. We relax this assumption to account for all distortions contracted due to simplification. Note that, stochasticity of the reward operator naturally cancels the assumptions of convexity or Lipshitz continuity.

To summarize, our key contributions are as follows.  
(a) We extend $\rho$-POMDP to probabilistic $\rho$-POMDP (\extendedpomdp) by relaxing the assumption that the reward operator and the belief update are deterministic; (b) We  introduce the general concept of  \ploss;  (c) We provide an online characterization of the \ploss;  (d) Finally, we exemplify our framework on a particular simplification technique, which is reducing the number of samples for planning. 


\section{Notations and Problem Formulation}
Let us denote by $\probd$ the probability density function and by $\prob$ the probability. By lowercase letter we denote a random vector or its realization.  For two random variables $x$ and $y$, we say that they are equal $x=y$ if they are equal as functions on their measurable space. Further, to shorten notations, we shall often use $\square_{k+}$ to denote $\square_{k+1:k+L}$, where $L$ is the planning horizon. By $\equiv$ we denote identity. 

\subsection{$\rho$-POMDP}
Let $k$ be an arbitrary time step.  
$\rho$-POMDP \cite{araya2010pomdp} is a tuple
\begin{equation}
	\langle \mathcal{X}, \mathcal{A}, \mathcal{Z}, T, O, \rho, \gamma, b_{0}\rangle, \label{eq:POMDP}
\end{equation} 
where $\mathcal{X}, \mathcal{A}, \mathcal{Z}$ are state, action, and observation spaces with $x_{k} \in \mathcal{X}, a_{k} \in \mathcal{A}, z_{k} \in \mathcal{Z}$ the momentary state, action, and observation,  respectively, $T(x_k,a_k,x_{k+1})= \mot{x_{k}}{a_{k}}{x_{k+1}} $ is the stochastic transition model from the past momentary state $x_k$ to the next  $x_{k+1}$ through action $a_k$, $O(z_k,x_k)=\observ{x_k}{z_k}$ is the stochastic observation model, $\rho\left(b_{k+1}, a_k\right)$ is a scalar belief and previous action dependent reward operator, $\gamma \in [0,1]$ is the discount factor, and $b_0$ is the belief about the initial state (prior). Throughout this paper we assume that $\gamma=1$. 

\subsection{Belief Space Planning}
The posterior belief at time instant $k$ is given by 
\begin{equation}\label{eq:belief}
	b_k(x_k) \approx \probd \left(x_k \sep b_0, a_{0:k-1}, z_{1:k} \right), 
\end{equation}
The belief is an efficient way of storing all relevant information that is obtainable so far. The usual assumption is that the belief is a sufficient statistic for decision making objective \cite{bertsekas2017dynamic}. However, in practice, the belief requires some representation. In general, this representation is not perfect, e.g.,~parametric or sampled form; thus, in \eqref{eq:belief}, we used the $\approx$ sign. In a real life scenario 
\begin{equation}	
	b_k = \psi(\psi(\dots \psi(b_0, a_0, z_1), a_{k-2}, z_{k-1} ), a_{k-1}, z_{k}),
\end{equation}
where $\psi$ is a method for updating the belief. 

Denote by $\pi_{\ell}$ policy at time step $\ell$ such that  $\pi_{\ell}(b_{\ell})=a_{\ell}$ maps belief to the action. It is noteworthy that policy $\pi(b)$ is a random function of the belief in general. For simplicity we assume that policy is deterministic. However, our development is not constrained to deterministic policies.  By $\pi \bydef \pi_{k:k+L-1}$ we denote a vector of policies for $L$ time steps starting from time step $k$.

To behave optimally, the robot shall choose a policy maximizing the objective, which in its general form, is 
\begin{eqnarray}
	&	J^L(b_k, \pi)= 
	\varphi \bigg(\probd \left(\rew_{k+1:k+L} \sep b_k, \pi_{k:k+L-1} \right), \ret \bigg) \label{eq:Obj}\\
	&\quad  \ \ \text{s.t. } b_{\ell}=\psi(b_{\ell-1}, \pi_{\ell-1}(b_{\ell-1}), z_{\ell}), \nonumber
\end{eqnarray}
where $L$ is a planning horizon, $\rew_{\ell}$ is a random reward, $\varphi$ is a projection operator, and 
$\ret \bydef f_{\ret}(\rew_{k+})$  is the return \cite{sutton2018reinforcement}. The return is some known function of the realization of  $\rew_{k+1:k+L}$; as discussed in \cite{Defourny08ws}, e.g.,~it could correspond to the cumulative reward $\ret = \sum_{\ell=1}^L \rew_{k+\ell}$.
If \eqref{eq:Obj} admits Bellman form, it can be written as
\begin{align}
	\label{eq:ObjBellman}
	&J^L(b_k, \pi_{k:k+L-1})=  \\
	&  J^1(b_k, \pi_{k}) + \int_{b_{k+1}} \probd(b_{k+1}|b_k,\pi_k) J^{L-1}(b_{k+1}, \pi_{k+1:k+L-1}) \nonumber \\
	&\text{s.t. } b_{\ell}=\psi(b_{\ell-1}, \pi_{\ell-1}(b_{\ell-1}), z_{\ell}). \nonumber
\end{align}
For example, a common choice for $\varphi$ is expectation over the distribution of future rewards given all data available \cite{Defourny08ws}.
However, our formulation considers a general projection operator.   

$\psi$ is a general  method for propagating the belief with action and updating it with the received observation. 
Sometimes the belief $b_{\ell-1}$ has a simple parametric form  $\theta_{\ell-1} $, where $\theta_{\ell-1}$ is the vector of parameters, e.g.,~a Gaussian belief. In this case, belief update $\psi$ can be deterministic, and is denoted by $\psi_{dt}(\theta_{\ell-1}, \pi_{\ell-1}(\theta_{\ell-1}), z_{\ell})$, where the subscript $dt$ stands for deterministic.  
In more general and challenging scenarios the belief $b_{\ell-1}$ is given by a set of weighted  samples  $ \{(w^i_{\ell-1}, x^i_{\ell-1})\}^N_{i=1}$. Therefore, $\psi$ is a stochastic method, e.g.,~a particle filter \cite{thrun2005probabilistic}. Applying multiple times $\psi$ on the same input will yield different sets of samples approximating the same distribution of the posterior belief. We denote the stochastic $\psi$ by $\psi_{st}(b_{\ell-1}, \pi_{\ell-1}(b_{\ell-1}), z_{\ell})$. Thus,  $\psi_{st}$ is a random function of the previous belief, an action and the observation.  
Note also another common situation where $b_{\ell-1}$ is parameterized, but there is no closed form update. In this case,  $\psi$ is also a stochastic method. 
Another form to formulate the above is that the distribution 
\begin{align}
	B(b_{\ell-1}, \pi_{\ell-1}(b_{\ell-1}),z_{\ell},b_{\ell}) \bydef \probd_B \left(b_{\ell} \sep b_{\ell-1}, \pi_{\ell-1}, z_{\ell}\right), \label{eq:StochBeliefUpdate}
\end{align}
is not a Dirac delta function. This aspect was disregarded so far, to the best of our knowledge. Note that in a belief-MDP formulation, the assumption is that $B$ is a Dirac delta function. We emphasize relation to belief-MDP  in Appendix~\ref{app:DiscussionBeliefMDP}.

Similar arguments also hold for the momentary reward operator of the belief and the previous action. In its pure theoretical form, the momentary reward is a  deterministic operator  of the posterior belief and possibly an action. For example, a common immediate reward is of the form 
\begin{align}
	\rew_{dt}(b) = \mathbb{E}_{x \sim b}\left[ f(b(x),x) \right]=\int_x b(x) f(b(x),x)dx, \label{eq:DetReward}
\end{align}  
where usually $f(b(x),x)= - \log b(x)$ or some reward on the state $f(b(x),x)= r(x)$, producing differential entropy or mean distance to goal. 
Unfortunately, an analytical expression for the reward operator $\rew_{dt}(\cdot)$ is available in only limited scenarios, e.g., if the belief is modeled as Gaussian and the reward is differential entropy. The representation of the beliefs in \eqref{eq:StochBeliefUpdate} dictates practical reward operators. Sometimes the deterministic operator can be constructed on top of a particular belief representation. E.g., \eqref{eq:StochBeliefUpdate} outputs a set of weighted samples and \eqref{eq:DetReward} is adapted to be a deterministic operator of this output \cite{boers2010particle}. However, it is not always possible. In extremely challenging situations the reward includes modification of the representation of the belief. This could introduce an additional source of stochasticity. 
We extend \eqref{eq:DetReward} to  
\begin{align}
	R(b_{\ell}, \pi_{\ell-1}(b_{\ell-1}), \rew_{\ell}) \bydef \probd_R \left(\rew_{\ell} \sep b_{\ell} ,\pi_{\ell-1}(b_{\ell-1})\right), \label{eq:StochRewardUpdate}
\end{align}  
embracing these possibilities.  
To our knowledge, we are the first who treat these aspects as random. 

\subsection{Problem Formulation}
Our goal is to generalize the concept of absolute action consistency to probabilistic, more lenient. Given some simplification method, we want to affirm \emph{online} what is the loss and provide probabilistic guarantees. Our approach to analysis and reasoning shall be agnostic to the choice of the projection operator $\varphi$.  
\section{Approach}
\begin{figure}[t]
	\centering 
	\includegraphics[width=\columnwidth]{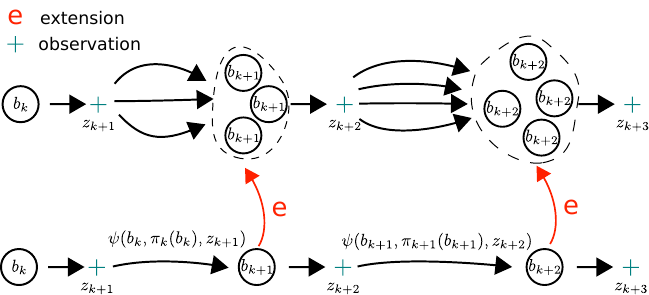}
	\caption{Illustration of one branch of the extended belief tree. In a conventional setting (bottom), under the policy $\pi$, a specific realization of observations $z_{k+1:k+3}$ defines the beliefs along the way. In our extended setting (top), that is not the case, as discussed in text. It is customary to choose the same beliefs used to build the tree to obtain reward distribution or samples from the reward. We decoupled beliefs from the tree and beliefs from the reward calculation. By the red arrow, we denote our extension (red e). } \label{fig:ExtendedBeliefTree}
\end{figure}
\subsection{Probabilistic $\rho$-POMDP}
To capture the complete simplification impact, we shall account for all potential sources of variability. We remove conventional approximations by extending \eqref{eq:POMDP} to a probabilistic reward model $R$ \eqref{eq:StochRewardUpdate} and probabilistic belief update $B$ \eqref{eq:StochBeliefUpdate}, and introduce
\begin{align}
	M=\langle \mathcal{X}, \mathcal{A}, \mathcal{Z}, T, O, R, \gamma, b_{k}, B\rangle, \label{eq:ExtendedPOMDP}	 
\end{align}
which we name probabilistic $\rho$-POMDP (\extendedpomdp). The rationale behind these conditional distributions ($R$ and $B$) is to capture additional sources of stochasticity, such as stochastic belief update, stochastic calculation of a given reward operator or simply not knowing the operator reward in explicit analytic form. 
These previously overlooked sources of stochasticity impact the likelihood of the observations 
\begin{align}
&\probd \left(z_{k+1:k+L} \sep b_k, \pi \right), \label{eq:ObsLikelihood}
\end{align} 
as well as the joint reward distribution $\probd \left(\rew_{k+} \sep b_k, \pi, z_{k+} \right) \equiv \probd \left(\rew_{k+1:k+L} \sep b_k, \pi_{k:k+L-1}, z_{k+1:k+L} \right)$ given a realization of future observations. The latter can be factorized as 
\begin{equation}
	\label{eq:ConditionalReward}
	\prod_{\ell=k+1}^{k+L} \int_{b_{\ell}} \probd_{R}  \left(\rew_{\ell} \sep b_{\ell}, \pi_{\ell-1}\right) \probd_{B} \left(b_{\ell}  \sep  b_{\ell-1}, \pi_{\ell-1}, z_{\ell}  \right) ,
\end{equation}
which is Dirac's delta function in the regular setting of POMDP and $\rho$-POMDP.  If $B$ is a Dirac function, a sample from \eqref{eq:ObsLikelihood} uniquely defines the corresponding posterior beliefs $b_{k+1:k+L}$. This, therefore, corresponds to the classical belief tree. In contrast, our \extendedpomdp \eqref{eq:ExtendedPOMDP}, corresponds to an \emph{extended} belief tree, which, due to  \eqref{eq:StochBeliefUpdate}, 
allows many samples of the beliefs $b_{k+1:k+L}$ for each sample of $z_{k+1:k+L}$ from \eqref{eq:ObsLikelihood}. We illustrate this in Fig.~\ref{fig:ExtendedBeliefTree}.

\subsection{Simplification Formulation} \label{sec:SimplificationFormulation}
To formally define the  simplification procedure, we augment the \extendedpomdp tuple \eqref{eq:ExtendedPOMDP} with a simplification operator $\nu$, 
\begin{equation}
	M_{\nu} =\langle \mathcal{S}, \mathcal{A}, \mathcal{Z}, T, O, R, \gamma, b_{k}, B, \nu\rangle,  \ \ \  \nu \bydef \nu_k, \dots, \nu_{k+L}. \label{eq:POMDPnu}
\end{equation} 
This general operator defines any possible modification of the original problem defined by \eqref{eq:ExtendedPOMDP} alongside with \eqref{eq:Obj} to a new, simpler to solve, problem. The operator $\nu$ can be for example, sparsification of the initial belief $b_k$ \cite{elimelech2019efficient}, substitution of the operator differential entropy by a simpler operator, e.g.,~trace of covariance matrix, discarding the normalizer in the differential entropy operator \cite{ryan2008information}, replacing the reward by its topological signature \cite{kitanov2019topological}. 

To distinct a simplified reward from the original reward, we denote the former by $\rewsimpl$ instead of $\rew$; similarly, we denote the 
simplified belief by $\bsimpl$ instead of $b$. Note the operator $\nu$ can be stochastic, as discussed below.

Specifically, belief simplification is described by the distribution
\begin{align}
	\probd( \bsimpl_{\ell} \sep b_{\ell}; \nu^{b}_{\ell}). \label{eq:Stochnu} 
\end{align}
In general, the distribution \eqref{eq:Stochnu} over the simplified belief $\bsimpl_{\ell}$ corresponds to a stochastic simplification operator $\nu^{b}_{\ell}$. This is the case, for example, when $b_{\ell}$ is represented by a set of $N$ weighted samples and $\nu_{\ell}^b$ is the operation of subsampling $n$ samples according to weights; i.e., applying this operation on $b_{\ell}$ multiple times leads to different sets of $n$ samples, each representing another realization of $\bsimpl_{\ell}$ from \eqref{eq:Stochnu}. For a deterministic operator $\nu_{\ell}^b$, $\eqref{eq:Stochnu}$ is a Dirac function. 

\begin{figure}[t] 
	\centering
	\begin{minipage}[t]{0.49\columnwidth}
	\centering 
	\includegraphics[width=\textwidth]{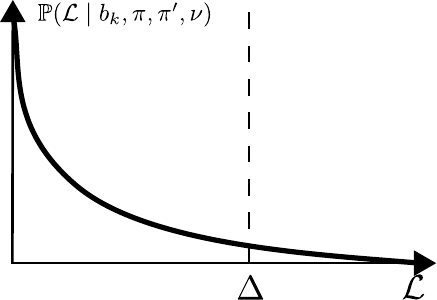}
	\subcaption{}
	 \label{fig:LossD}
    \end{minipage}%
	\hfill
	\begin{minipage}[t]{0.49\columnwidth}
	\centering 
	\includegraphics[width=\textwidth]{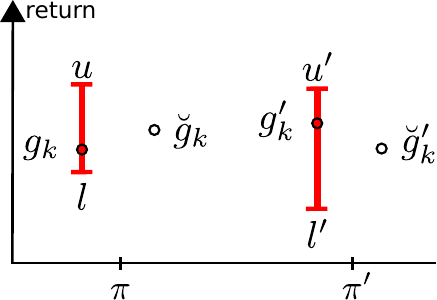}
	\subcaption{}
	\label{fig:RewardsSamples}
    \end{minipage} 
	\caption{Illustration of \textbf{(a)} the distribution of loss, and \textbf{(b)} the online bounds of the return.} 
\end{figure}

Further, there are several cases of how a simplification affects belief update \eqref{eq:StochBeliefUpdate} from time $\ell-1$ to $\ell$. 
\begin{enumerate}
	\item Without any simplification we have $\probd_B(b_{\ell} \sep b_{{\ell}-1},\pi_{\ell-1},z_{\ell})$ from \eqref{eq:StochBeliefUpdate}.
	\item Given a simplified belief $\bsimpl_{\ell-1}$, while keeping the original stochastic belief update $\psi_{st}$, we have $\probd_{B}(\bsimpl_{\ell} \sep \bsimpl_{\ell-1},\pi_{\ell-1},z_{\ell})$, where each realization of $\bsimpl_{\ell}$ is obtained via $\psi_{st}$. Thus, given $\bsimpl_{{\ell}-1}$, this  distribution is not a function of $\nu$.
	\item We can also simplify the belief update operator, $\psi_{st}$, to $\breve{\psi}_{st}$. Denoting the corresponding simplification operator   $\nu_{\ell}^{\psi}$, this yields  $\probd_{\breve{B}}(\bsimpl_{\ell} \sep \bsimpl_{\ell-1},\pi_{\ell-1},z_{\ell}; \nu_{\ell}^{\psi})$.
	\item Finally, one can decide at time $\ell$ to apply simplification on the belief (determined by $\nu_{\ell}^b$) via \eqref{eq:Stochnu}. 
	The corresponding belief update can be written as
 	$\probd_{\breve{B}}(\bsimpl_{\ell} \sep \bsimpl_{\ell-1},\pi_{\ell-1},z_{\ell}; \nu_{\ell}^b, \nu^{\psi}_{\ell})=\int_{\tilde{b}_{\ell}} \probd( \bsimpl_{\ell} \sep  \tilde{b}_{\ell}; \nu_{\ell}^b) \probd_{\breve{B}}( \tilde{b}_{\ell} \sep \bsimpl_{\ell-1},\pi_{\ell-1},z_{{\ell}}; \nu_{\ell}^{\psi})$, where  $\tilde{b}_{\ell}$ is the integration variable.  
\end{enumerate}
We combine these cases and write 
\begin{align}
&\breve{B}\left(\!\bsimpl_{\ell-1}, \pi_{\ell-1},z_{\ell},\bsimpl_{\ell}; \nu \! \right)  \!\bydef \! 
\probd_{\!\breve{B}}(\bsimpl_{\ell} \sep \bsimpl_{\ell-1},\pi_{\ell-1},z_{\ell}; \nu_{\ell}^b, \nu^{\psi}_{\ell}).
 \label{eq:SimplStochBeliefUpdate} 
\end{align} 
Similarly, reward simplification could be, in general, stochastic, leading to the distribution
\begin{align}
\probd( \rewsimpl_{\ell} \sep \rew_{\ell}; \nu^{\rho}_{\ell}). \label{eq:Stochnureward} 
\end{align}  
Thus, given a simplified belief $\bsimpl_{\ell}$ and $\bsimpl_{\ell-1}$, and recalling \eqref{eq:StochRewardUpdate}, the distribution over $\rewsimpl_{\ell}$ is
$$\probd_{\breve{R}}(\rewsimpl_{\ell} \sep \bsimpl_{\ell}, \pi_{\ell-1}(\bsimpl_{\ell-1}); \nu) \! = \!\! \int_{\tilde{\rew}_{\ell}} \! \! \! \! \probd( \rewsimpl_{\ell} \sep \tilde{\rew}_{\ell}; \nu^{\rho}_{\ell}) \probd_{R}( \tilde{\rew}_{\ell} \sep \bsimpl_{\ell}, \pi_{\ell-1}(\bsimpl_{\ell-1})),$$
which we denote as the simplified reward model, 
\begin{align}
&\breve{R}(\bsimpl_{\ell}, \pi_{\ell-1}(\bsimpl_{\ell-1}), \rewsimpl_{\ell}; \nu) \bydef \probd_{\breve{R}} \left(\rewsimpl_{\ell} \sep \bsimpl_{\ell}, \pi_{\ell-1}(\bsimpl_{\ell-1}); \nu \right). \label{eq:SimplStochRewardUpdate} 
\end{align} 
Consequently,  the models \eqref{eq:SimplStochBeliefUpdate} and \eqref{eq:SimplStochRewardUpdate} impact \eqref{eq:ConditionalReward}, and lead to the following simplified joint reward distribution given a realization of future observations 
\begin{align}
\label{eq:ConditionalJointTimeReward}
&\probd \left(\rewsimpl_{k+} \sep b_k, \pi, z_{k+}, \nu \right)  = \int_{\bsimpl_{k}} \probd( \bsimpl_{k} \sep b_{k}; \nu^{b}_{k}) \prod_{\ell=k+1}^{k+L} \int_{\bsimpl_{\ell}}   \\
&\probd_{\breve{R}} \left(\rewsimpl_{\ell} \sep \breve{b}_{\ell}, \pi_{\ell-1}; \nu\right)\probd_{\breve{B}} \left(\bsimpl_{\ell} \sep \bsimpl_{\ell-1}, \pi_{\ell-1}, z_{\ell}; \nu \right) . \nonumber
\end{align}
The formulations of \eqref{eq:SimplStochBeliefUpdate} and \eqref{eq:SimplStochRewardUpdate} assume that inputs to $\breve{B}$ and $\breve{R}$ are most simplified versions at appropriate time instant $(\rewsimpl_{\ell}, \bsimpl_{\ell}, \bsimpl_{\ell-1})$. This simplification approach uses only the observations from the belief tree. In the sequel, we explain why it is advantageous. In this setting, in addition to maintaining/updating $b_{\ell}$, we also have to maintain/update the simplified version $\bsimpl_{\ell}$.  

Alternatively, the update of simplified belief $\bsimpl_{\ell}$ can be avoided. Such simplification builds upon samples from $B$, which are already present at the belief tree. Thus,  
\begin{align}
\label{eq:ConditionalJointTimeReward2}
&\probd \left(\rewsimpl_{k+} \sep b_k, \pi, z_{k+}, \nu \right) = \int_{\bsimpl_{k}} \probd( \bsimpl_{k} \sep b_{k}; \nu^{b}_{k})  \prod_{\ell=k+1}^{k+L} \int_{b_{\ell}} \int_{\bsimpl_{\ell}}\\
&  \probd_{\breve{R}}(\rewsimpl_{\ell} \sep \bsimpl_{\ell}, \pi_{\ell-1}(\bsimpl_{\ell-1}); \nu)  \probd( \bsimpl_{\ell} \sep b_{\ell}; \nu)  \probd_{B} \left(b_{\ell} \sep b_{\ell-1}, \pi_{\ell-1}, z_{\ell}\right). \nonumber
\end{align}
To approximate \eqref{eq:ConditionalJointTimeReward2} one can use original beliefs from the extended belief tree or sample again if needed. 

In other words, in this paper we assume that operator $\nu$ affects exclusively \eqref{eq:ConditionalReward}. However, the measurements are sampled as in the original problem as in \eqref{eq:ObsLikelihood}.

\subsection{Probabilistic Loss (\ploss)}
In the previous section, we defined a simplification procedure that results in a corresponding new decision making problem that should be easier to solve. Now we stipulate on the quality of the simplification for two candidate policies $\pi$ and $\pi'$.

From $\probd \left(\rew_{k+} \sep b_k, \pi, z_{k+}, \nu \right)$  and $\probd \left(\rewsimpl_{k+} \sep b_k, \pi, z_{k+}, \nu \right)$ we arrive at the distribution of the  original as well as simplified returns  $\probd (\ret  \sep  b_k, \pi)$ and $\probd (\retsimpl  \sep  b_k, \pi, \nu)$ for each evaluated candidate policy.  
To quantify the impact of the simplification procedure, we shall consider the \emph{joint} distribution $\probd (\ret, \ret', \retsimpl, \retsimpl'  \sep  b_k, \pi, \pi', \nu)$. 
Our goal is to examine how the simplification procedure alters the joint distribution $\probd (\ret, \ret' \sep  b_k, \pi, \pi')$  towards $\probd (\retsimpl, \retsimpl'  \sep  b_k, \pi, \pi', \nu)$.  These two marginal distributions are illustrated in Fig.~\ref{fig:JointBeliefRewards}.


Let us define the following random variable, which we shall refer to as "loss"
\begin{equation}
	\mathcal{L} \bydef f_{\mathcal{L}}(\ret,\ret',\retsimpl, \retsimpl') = 
	\begin{cases} \max \{\ret'- \ret, 0 \} &\mbox{if } \retsimpl  -\retsimpl' >0,  \\
	 \max \{ \ret- \ret', 0\} &\mbox{if } \retsimpl  - \retsimpl' <0,\\
	 0 &\mbox{else. }\end{cases} \label{eq:plossExpr}
\end{equation} 
With \eqref{eq:plossExpr} we aim to capture a complete impact of a  simplification onto the decision making problem. 
Specifically, this definition captures for each possible realization of $\ret, \ret', \retsimpl, \retsimpl'$  the absolute difference between the original returns $\Delta=|\ret'- \ret|$ in case action trend was not preserved on this realization. Meaning, at this realization, the optimal actions of original and simplified problems would differ. Given a sample $(\ret, \ret', \retsimpl, \retsimpl')$, the simplification is action consistent at this sample if the sign of the difference of the returns is preserved. In other words, the same action would be identified as optimal with the original and simplified returns; else we must account for the loss \eqref{eq:plossExpr}. 

Our object of interest is the distribution  density of $\mathcal{L}$ given all the information available  at our disposal, 
\begin{align}
	\probd \left(\mathcal{L} \sep b_k, \pi, \pi', \nu \right). \label{eq:ploss}
\end{align} 
We denote this distribution  by Probabilistic Loss (\ploss), as it generalizes the concept of absolute action consistency to  probabilistic. See illustration in Fig.~\ref{fig:LossD}. E.g., if \eqref{eq:ploss} is the Dirac delta function $\delta(\mathcal{L})$, the simplification method is absolute action consistent for every possible operator projection $\varphi$. 

Moreover, for any $\Delta$, its cumulative distribution function (CDF) $\probd \left(\mathcal{L} \leq \Delta \sep b_k, \pi, \pi', \nu \right)$ provides probability to suffer loss at most $\Delta$. 
Similarly, the tail distribution function (TDF) $\probd \left(\mathcal{L} > \Delta \sep b_k, \pi, \pi', \nu \right)$ provides probability to suffer loss greater than $\Delta$. 
We shall revisit and discuss these aspects further in Section \ref{sec:CharacterOnline}.

\subsection{Decomposition of Returns}\label{subsec:DecomposeRewards}

The source of distribution  \eqref{eq:ploss} is 
\begin{align}
	& \probd (\ret, \ret', \retsimpl, \retsimpl'  \sep  b_k, \pi, \pi', \nu), \label{eq:QuadripleRewDist}
\end{align}
i.e.,~the joint distribution over original and simplified returns of both policies. This distribution 
decomposes via marginalization over future observations $z_{k+}\equiv z_{k+1:k+L}$  and $z'_{k+}\equiv z'_{k+1:k+L}$ as 
\begin{eqnarray}
	\int_{{\substack{z_{k+} \\ z'_{k+} }}}  \probd (\ret, \ret', \retsimpl, \retsimpl'  \sep  b_k, \pi, \pi', \nu, z_{k+}, z'_{k+} ) \cdot 
	\\ 
	\ \ \ \ \probd (z_{k+}, z'_{k+} \sep b_k, \pi, \pi') dz_{k+} dz'_{k+},
\nonumber
\end{eqnarray}
which, according to \eqref{eq:StochBeliefUpdate}, \eqref{eq:StochRewardUpdate} and \eqref{eq:SimplStochBeliefUpdate}-\eqref{eq:SimplStochRewardUpdate}, decomposes to
\begin{eqnarray}
	\int_{{\substack{z_{k+} \\ z'_{k+} }}}  \probd (\ret,  \retsimpl \sep  \his_{k+L}, \nu) \probd (\ret', \retsimpl'  \sep  \his'_{k+L}, \nu) \cdot \label{eq:JointRewardsDecomposed}
	\\ 
	\ \ \ \ \probd (z_{k+}, z'_{k+} \sep b_k, \pi, \pi') dz_{k+} dz'_{k+},
	\nonumber
\end{eqnarray}
where $\his_{k+L} \bydef \{b_k, \pi, z_{k+}\}$ and  $\his'_{k+L} \bydef \{b_k, \pi', z'_{k+}\}$. Note that the belief $b_k$  is shared by both histories.

In other words, the simplification operator $\nu$ independently affects  each realization of the future. 
Given two such realizations $(\his_{k+L}, \his'_{k+L}, \nu)$, the pairs of original and simplified returns are statistically independent of all other rewards.  This crucial observation will be significant in the sequel. 
\begin{figure}[t]
	\centering 
	\includegraphics[width=\columnwidth]{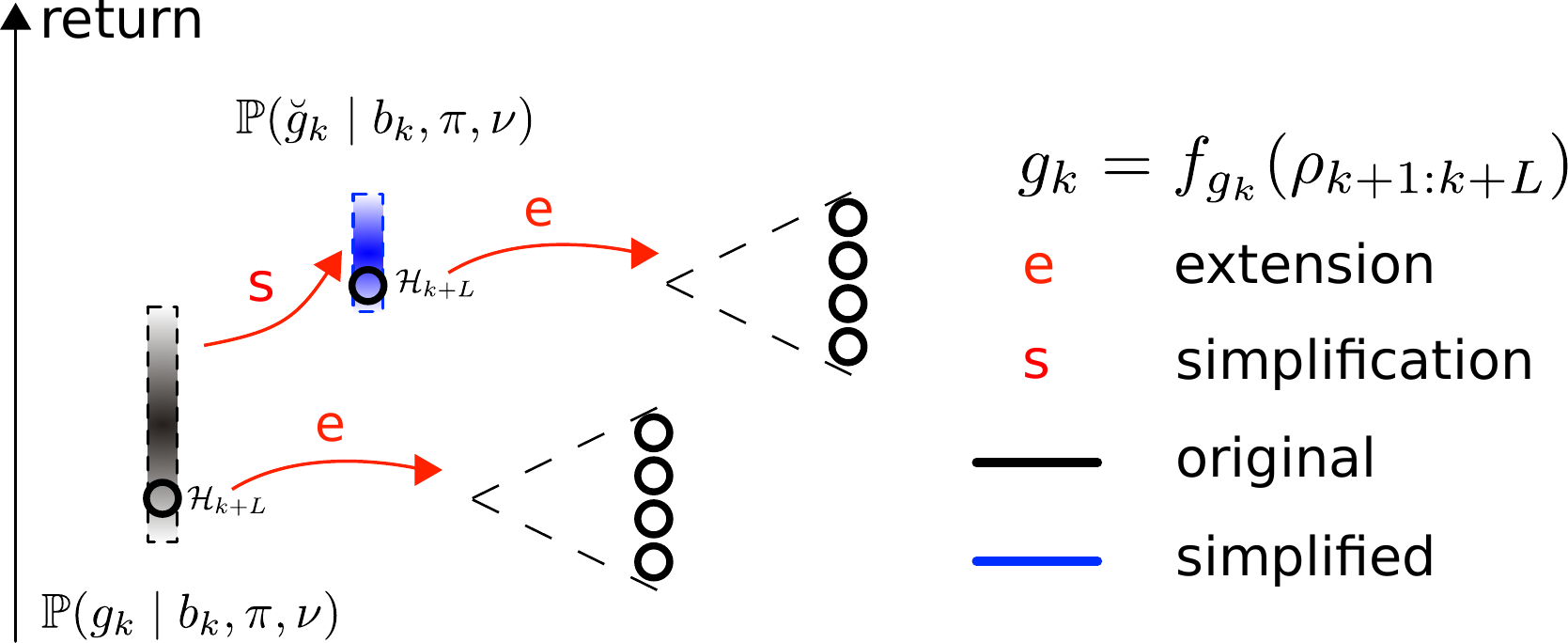}
	\caption{ Our extended setting permits variability of the reward given the present and a realization of the future. On the contrary, in a conventional setting, \eqref{eq:RewSimplRewCond} is always a Dirac delta function.} \label{fig:Extension}
\end{figure}
\subsection{Online Bound on Probabilistic Loss (\pbloss)}

The distribution defined by \eqref{eq:ploss} requires  access to \eqref{eq:QuadripleRewDist}
which we do not have in an online setting. To circumvent the requirement of accessing $\ret$ and $\ret'$, we propose to substitute  them by online lower and upper bounds  $\lb, \ub$ and $\lb', \ub'$, respectively. These bounds should be accessible without knowledge of original returns.

Let us consider a sampled return realization $(\ret, \ret', \retsimpl, \retsimpl') \sim \probd (\ret, \ret', \retsimpl, \retsimpl'  \sep  b_k, \pi, \pi', \nu)$ from \eqref{eq:QuadripleRewDist}.  As in an online setting we do not actually have access to the original returns $(\ret, \ret')$, we strive to bound the latter, 
\begin{align}
	\lb\leq \ret \leq \ub, \ \lb' \leq \ret' \leq \ub', \label{eq:Bounds}
\end{align}
where, for now, we assume \eqref{eq:Bounds} holds for any sample of $(\ret, \ret', \retsimpl, \retsimpl')$; for example, these could be analytically-derived bounds. This setting is illustrated in Fig.~\ref{fig:RewardsSamples}. However, in Section \ref{subsec: online_stoch_bounds} we also discuss a more general setting where we  allow \eqref{eq:Bounds} to be violated with probability larger than zero.

Using these bounds we are able to define online a bound on loss \eqref{eq:plossExpr} \emph{without} accessing the original problem ($R$ and $B$),  
\begin{align}
	&\bar{\mathcal{L}} \bydef \! f_{\bar{\mathcal{L}}}(\retsimpl, \lb, \ub, \retsimpl', \lb' , \ub') \! = \!\! \begin{cases} \max \{\ub' - \lb ,0 \} &\mbox{if } \retsimpl - \retsimpl' >0, \\
		\max \{ \ub - \lb' ,0 \} &\mbox{if } \retsimpl -  \retsimpl' <0, \\
		0 &\mbox{else. }\end{cases} \label{eq:pblossExpr}
\end{align}
Note that sometimes we can find bounds over the returns by applying the same function $f_{\ret}$ on the bounds on the momentary rewards (returns when $L=1$), e.g., in case of cumulative reward $	\ub = \sum_{\ell =k+1}^{k+L} \ub_{\ell}$ and $\lb = \sum_{\ell =k+1}^{k+L} \lb_{\ell}$. 
However, this is not always possible, e.g.,~if $\ret$ deviates from the sum of momentary rewards or in the case of  Bellman form \eqref{eq:ObjBellman}. Sometimes it is, therefore, better to work with momentary bounds. 

In an online setting, we are interested in the distribution density of $\bar{\mathcal{L}}$, 
\begin{align}
	\probd \left(\bar{\mathcal{L}} \sep b_k, \pi, \pi', \nu \right), \label{eq:pbloss}
\end{align} 
which we denote by Probabilistic Bound on Loss (\pbloss). 

As we discuss in Section \ref{sec:CharacterOnline}, \pbloss characterizes the impact of a  simplification in an online setting; thus, it enables to determine online if  a candidate simplification is acceptable given a user-specified criteria. The decision to either accept or decline a (candidate) simplification is guided by probabilistic guarantees, as provided by our approach.

\subsection{Online Stochastic Bounds} \label{subsec: online_stoch_bounds}
Our extension  allows $R$ and $B$, as well as $\breve{R}$ and $\breve{B}$ to be any distributions. They can remain Dirac functions,  e.g.,~if belief update and the reward calculation have a closed form
\begin{align}
\probd \left(\rew_{\ell} \sep b_{\ell-1}, a_{\ell-1}, z_{\ell} \right) \!=\!\delta(\rew_{\ell} \!-\!\rew_{dt}(\psi_{dt}( \theta_{\ell-1}, a_{\ell-1}, z_{\ell}))). \label{eq:Degenerate}
\end{align}
In such a case, conditioned on $(\his_{k+L}, \his_{k+L}', \nu)$, the returns $(\ret, \ret', \retsimpl, \retsimpl')$ are deterministic. 

However, in the more general case, following our extension, there is a joint distribution of original and simplified returns given a realization of the future and the present,
\begin{align}
	\probd (\ret, \retsimpl \sep  \his_{k+L}, \nu), \label{eq:RewSimplRewCond}
\end{align}	 
as illustrated in Fig.~\ref{fig:Extension}. Since \eqref{eq:RewSimplRewCond} is no longer a Dirac function, we can use knowledge about this distribution to design bounds, which will hold with \emph{some} probability.

In section \ref{sec: spec}, we show that it is possible to harness the structure of \eqref{eq:RewSimplRewCond} to design the mentioned more  lenient online bounds. 
Our framework permits to detach the process of estimation of the bounds from the realization of the reward and truly use all accessible information in a simplified problem. For example, one way to design probabilistic bounds is to find online a random variable $\epsilon$ such that the probability 
\begin{align}
\probd (|\ret - \retsimpl | \leq \epsilon \sep  \his_{k+L}, \nu)
\end{align}
is bounded from below. The corresponding probabilistic lower and upper bounds will be $\lb = \retsimpl - \epsilon$ and $\ub = \retsimpl +\epsilon$, respectively. We, therefore, refer to $\lb$ and $\ub$ as random variables.
In our setting, even if the bounds actually bound with very low probability, it is still possible to analyze the quality of the simplification.  Moreover, the analytical bounds,  designed in a conventional setting,  can be used in our extended setting without any revision. In our extended environment, they will bound with probability one.

Having introduced the novel stochastic bounds, we proceed to the formulation of the constraints, that these bounds shall fulfill to be meaningful. Our goal is to formulate conditions, which assure that \pbloss \eqref{eq:pbloss} is connected to \ploss \eqref{eq:ploss} and can be used online to analyze the quality of the simplification.   

The following conditional
\begin{align}
& \probd (\ret , \ret', \retsimpl, \retsimpl', \lb, \ub, \lb', \ub' \sep  \his_{k+L}, \his'_{k+L}, \nu), \label{eq:GeneratingDistr}
\end{align}
encloses all the variables situated in the problem. Moreover, following Section \ref{subsec:DecomposeRewards}, \eqref{eq:GeneratingDistr} decomposes into
\begin{align}
\probd (\ret, \retsimpl,  \lb, \ub    \sep \his_{k+L},  \nu) \probd (\ret', \retsimpl', \lb', \ub'   \sep \his'_{k+L}, \nu). 
\end{align}
For calculating \pbloss \eqref{eq:pbloss} we will need samples from $\probd (\retsimpl,  \lb, \ub    \sep \his_{k+L},  \nu)$ and $\probd (\retsimpl',  \lb', \ub'    \sep \his'_{k+L},  \nu)$. Let $[\cdot]$ be the Iverson bracket and $\alpha \in [0, 1)$. For every possible sample $(\retsimpl, \retsimpl')$ we do not know which sample $(\ret, \ret')$ one could obtain in the original problem. However, if the bounds are designed such that    
\begin{align}
	\probd (\ret, \lb, \ub \sep  \his_{k+L}, \nu) \quad 	, \quad \probd (\ret', \lb', \ub' \sep  \his'_{k+L}, \nu)
	 \label{eq:RewLbUbCond}
\end{align} 
render
\begin{align}
(1-\alpha) \leq \prob \left(\left[\lb \leq \ret \leq \ub \right] = 1 \sep   \his_{k+L}, \nu\right)   \label{eq:LenientConstraint1}
\end{align}
and  
\begin{align}
(1-\alpha)\leq \prob \left(\left[\lb' \leq \ret' \leq \ub'\right]  = 1 \sep  \his'_{k+L}, \nu \right),  \label{eq:LenientConstraint2}
\end{align}
we can bound online CDF and TDF of \ploss using \pbloss, as we show in Section \ref{sec:CharacterOnline}.

We note that, in general, \eqref{eq:LenientConstraint1} and \eqref{eq:LenientConstraint2} could each have its own $\alpha$. 
All developments and proofs can be adjusted easily to this setting.

Considering the above,  \pbloss \eqref{eq:pbloss} is based on 
\begin{align}
\probd (\retsimpl, \lb, \ub, \retsimpl', \lb' , \ub'  \sep  b_k, \pi, \pi', \nu). \label{eq:GeneratingOnlineMarginal}
\end{align}
To summarize, there are three types of online reward bounds:
\begin{enumerate}
	\item Deterministic bounds. These analytical bounds exist in case of a closed form belief update $\psi_{dt}$ and a deterministic operator reward $\rew_{dt}(b)$ from \ref{eq:DetReward}, e.g.,~belief is a Gaussian and the reward is differential entropy. In this case, even in our extended setting $R$ and $B$ remain Dirac functions. 
	\item Stochastic bounds that hold with probability one. These are also analytical bounds. In our extended setting $R$ and $B$ are no longer Dirac functions. However, these bounds hold for any realization of sample approx., as stated around \eqref{eq:Bounds}. 
	\item Stochastic bounds that hold at least with probability $1-\alpha$. They exist only in our extended setting when $R$ and $B$ are not Dirac functions.
\end{enumerate}

\subsection{Characterization of \ploss Online} \label{sec:CharacterOnline}
\begin{algorithm}[t]
	\begin{algorithmic}
		\caption{Online characterization of the simplification}
		\label{alg: the_alg}
		\STATE {\bfseries Input:} Two candidate policies $\pi, \pi'$. Initial belief $b_k$. Samplers from $\probd (\retsimpl, \lb, \ub \sep  \his_{k+L}, \nu)$ and  $\probd (\retsimpl', \lb', \ub' \sep  \his'_{k+L}, \nu)$.
		\STATE Sample $b_k$ or take the initial samples from inference. 
		Obtain $\probd \left(z_{k+1:k+L}, z'_{k+1:k+L} \sep b_k, \pi, \pi'\right)$ and create two belief policy trees.
		\FORALL {sample pairs $(z_{k+1:k+L}, z'_{k+1:k+L})$}
		\STATE Obtain sample $(\retsimpl, \lb, \ub, \retsimpl', \lb' , \ub')$. 
		\STATE Calculate $f_{\bar{\mathcal{L}}}(\retsimpl, \lb, \ub, \retsimpl', \lb' , \ub')$ according to \eqref{eq:pblossExpr}.
		\ENDFOR
		\STATE $\{ f_{\bar{\mathcal{L}}}(\retsimpl, \lb, \ub, \retsimpl', \lb' , \ub')\}$ represents the set of samples of $\bar{\mathcal{L}}$.
		\STATE {\bfseries Output:} $\forall \Delta$ empirically calculated $\prob \left(\bar{\mathcal{L}} > \Delta \sep b_k, \pi, \pi', \nu \right)$ as $\frac{\text{number of samples of }\bar{\mathcal{L}} \text{, satisfying } \bar{\mathcal{L}} > \Delta }{\text{number of all samples of } \bar{\mathcal{L}}}$. 
	\end{algorithmic}
\end{algorithm}

In this section, we show how  \pbloss can be used in an online setting to characterize \ploss (which is unavailable online). In turn, this enables to provide online probabilistic performance guarantees for a considered simplification (represented by operator $\nu$), or to decide if it is adequate given a user-specified criteria.

Specifically, recall \ploss CDF and TDF, i.e.,~probability to suffer loss at most, or greater, than $\Delta \in \mathbb{R}$, respectively, 
\begin{eqnarray}
	\text{\ploss CDF:} && \prob \left(\mathcal{L} \leq \Delta \sep b_k, \pi, \pi', \nu \right) \label{eq:CDF}
	\\
	\text{\ploss TDF:} && \prob \left(\mathcal{L} > \Delta \sep b_k, \pi, \pi', \nu \right). \label{eq:TDF}
\end{eqnarray}
We now aim to bound \ploss CDF \eqref{eq:CDF} from below, and \ploss TDF \eqref{eq:TDF} from above by utilizing \pbloss.

We now consider \ploss TDF and express $\prob \left(\mathcal{L} > \Delta \sep b_k, \pi, \pi', \nu \right)$ as
\begin{align*}
\prob \left(\mathcal{L} > \Delta, \bar{\mathcal{L}} \geq \mathcal{L} | b_k, \pi, \pi', \nu\right)\!+ \! \prob \left(\mathcal{L} > \Delta, \bar{\mathcal{L}} < \mathcal{L} | b_k, \pi, \pi', \nu\right).
\end{align*}
The first term can be written via chain rule as
\begin{align}
\prob \left(\mathcal{L} > \Delta | \bar{\mathcal{L}} \geq \mathcal{L},  b_k, \pi, \pi', \nu\right)\prob \left(\bar{\mathcal{L}} \geq \mathcal{L} | b_k, \pi, \pi', \nu\right).
\end{align}
Performing chain rule similarly also on the second term and recalling 
that  $\prob(\bar{\mathcal{L}} \geq \mathcal{L} \sep  \cdot ) + \prob(\bar{\mathcal{L}} < \mathcal{L} \sep  \cdot )=1$,  allows to express \ploss TDF as
\begin{eqnarray}
	&&\prob \! \left(\mathcal{L} > \Delta \sep b_k, \pi, \pi', \nu \right) \! = \!
	\prob \! \left(\mathcal{L} > \Delta | \bar{\mathcal{L}} \geq \mathcal{L},  b_k, \pi, \pi', \nu\right)\lambda 
	\nonumber \\ 
	&&\ \ \ + \prob \! \left(\mathcal{L} > \Delta | \bar{\mathcal{L}} < \mathcal{L}, b_k, \pi, \pi', \nu\right)(1-\lambda ),
\end{eqnarray}
where 
\begin{equation}
	\lambda \bydef \prob \left(\bar{\mathcal{L}} \geq \mathcal{L} | b_k, \pi, \pi', \nu\right) \equiv \prob \left([\bar{\mathcal{L}} \geq \mathcal{L}]=1 |  b_k, \pi, \pi', \nu \right) \label{eq:blossboundsloss}.
\end{equation}
While $\lambda$ from \eqref{eq:blossboundsloss} is unavailable, we can bound it from below using \eqref{eq:LenientConstraint1} and \eqref{eq:LenientConstraint2} as follows.  
%
\begin{thm}[Probability that bound bounds] \label{thrm:ProbBounBounds}
Fix $\alpha \in \mathbb{R}$.  Assume that \eqref{eq:LenientConstraint1} and \eqref{eq:LenientConstraint2} hold. Then:
\begin{align}
		&\prob \left([\bar{\mathcal{L}} \geq \mathcal{L}]=1 \sep  b_k, \pi, \pi', \nu \right)   \geq (1-\alpha)^2. \label{eq:prbBoundBounds}
\end{align}
\end{thm}
Proof: see Appendix~\ref{app:Proofs}. 

Now we show that given the event $\bar{\mathcal{L}} \geq \mathcal{L}$, \ploss TDF is bounded from above by \pbloss TDF.  
\begin{thm}[Conditional TDF Lower bound] $\forall \Delta \in \mathbb{R}$,
\begin{align}		
	\prob \left( \mathcal{L} > \Delta \sep \bar{\mathcal{L}} \geq \mathcal{L} , b_k , \pi, \pi', \nu\right) \! \leq  \!
	\prob \left(\bar{\mathcal{L}} > \Delta \sep \bar{\mathcal{L}} \geq \mathcal{L} , b_k, \pi, \pi', \nu \right). \nonumber
\end{align}
\end{thm} 
Proof: see Appendix~\ref{app:Proofs}.

Finally, we characterize \ploss  as follows. 
\begin{thm} [Upper and Lower bounds]
Denote $\beta(\Delta) \triangleq \min \left\{1, \frac{\prob \left(\bar{\mathcal{L}} > \Delta \sep b_k, \pi, \pi', \nu \right)}{(1-\alpha)^2} +2\alpha - \alpha^2 \right\}$, so
\begin{align}
	\prob \left(\mathcal{L} > \Delta \sep b_k, \pi, \pi', \nu\right)  \leq  \beta(\Delta)
\end{align}
and 
\begin{align}
	&\prob \left(\mathcal{L} \leq \Delta \sep b_k, \pi, \pi', \nu\right)  \geq 1 -  \beta(\Delta) .
\end{align}
\end{thm}	
Proof: see Appendix~\ref{app:Proofs}.

We can say that a simplification procedure, in a worst case scenario,  will render decision making sub optimal  at most $\Delta$ with probability at least $1-\beta(\Delta)$. Moreover, since  $0 \leq \mathcal{L}$, 
setting $\Delta=0$ in Alg.~\ref{alg: the_alg} we can assess the probability to be absolute action consistent in worst case scenario (for any $\varphi$). 

\subsection{Calculating \ploss Offline and \pbloss Online}
We discuss the offline calculation of the \ploss in Appendix~\ref{app:Calc_ploss_offline}.

So far, we did not explain how to calculate \pbloss \eqref{eq:pbloss}.  One approach is to sample $(\retsimpl, \lb, \ub, \retsimpl', \lb' , \ub')$ from \eqref{eq:GeneratingOnlineMarginal} and evaluate $\bar{\mathcal{L}}$ for each such sample via \eqref{eq:pblossExpr}. Then, \pbloss is represented by $\{ f_{\bar{\mathcal{L}}}(\retsimpl, \lb, \ub, \retsimpl', \lb' , \ub')\}$.

Generating samples from \eqref{eq:GeneratingOnlineMarginal} involves marginalizing over future measurements $z_{k+}\equiv z_{k+1:k+L}$  and $z'_{k+}\equiv z'_{k+1:k+L}$. Similar to \eqref{eq:JointRewardsDecomposed}, $\probd (\retsimpl, \lb, \ub, \retsimpl', \lb' , \ub'   \sep b_k, \pi, \pi', \nu)$ decomposes to
\begin{eqnarray}
	\int_{{\substack{z_{k+} \\ z'_{k+} }}}  \probd ( \retsimpl, l,u \sep  \his_{k+L}, \nu) \probd ( \retsimpl', l', u'  \sep  \his'_{k+L}, \nu) \cdot \label{eq:SimplifiedJointRewardsDecomposed}
	\\ 
	\ \ \ \ \probd (z_{k+}, z'_{k+} \sep b_k, \pi, \pi') dz_{k+} dz'_{k+},
	\nonumber
\end{eqnarray}
In practice, $ \probd (z_{k+}, z'_{k+} | b_k, \pi, \pi')$  corresponds to two extended belief policy trees, starting from the same root ($b_k$) and having the same rule for choosing rollouts. The specific way of  obtaining samples  from  
\begin{align}
\probd (\retsimpl, \lb, \ub \sep  \his_{k+L}, \nu) \quad 	, \quad \probd (\retsimpl', \lb', \ub' \sep  \his'_{k+L}, \nu) \label{eq:OnlineMarginals}
\end{align} 
depends on the operator $\nu$. In the next section, we  elaborate on these aspects, considering a specific simplification operator.
\begin{figure}[t] 
	\centering
	\begin{minipage}[t]{0.49\columnwidth}
		\centering 
		\includegraphics[width=\textwidth]{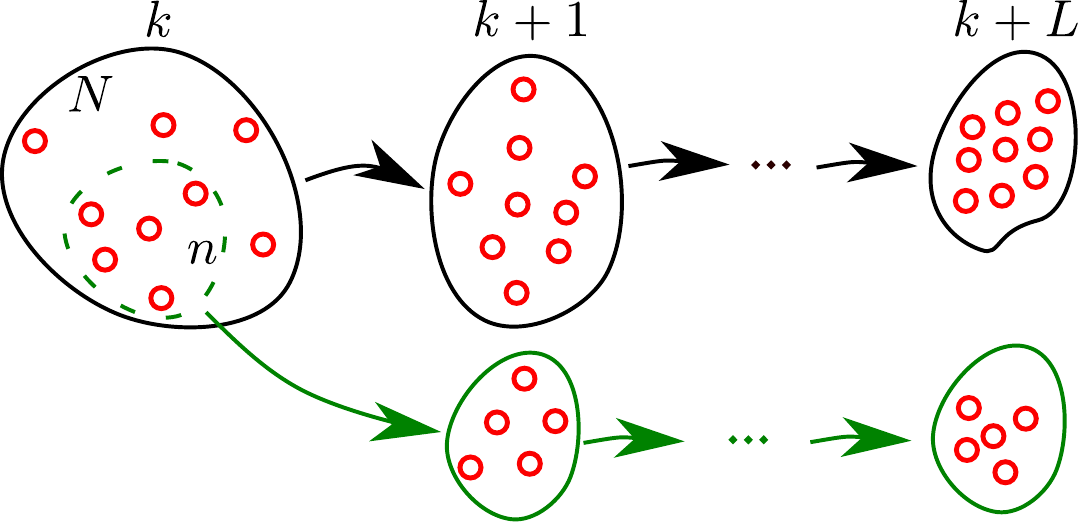}
		\subcaption{}
	\end{minipage}%
	\hfill
	\begin{minipage}[t]{0.49\columnwidth}
		\centering 
		\includegraphics[width=\textwidth]{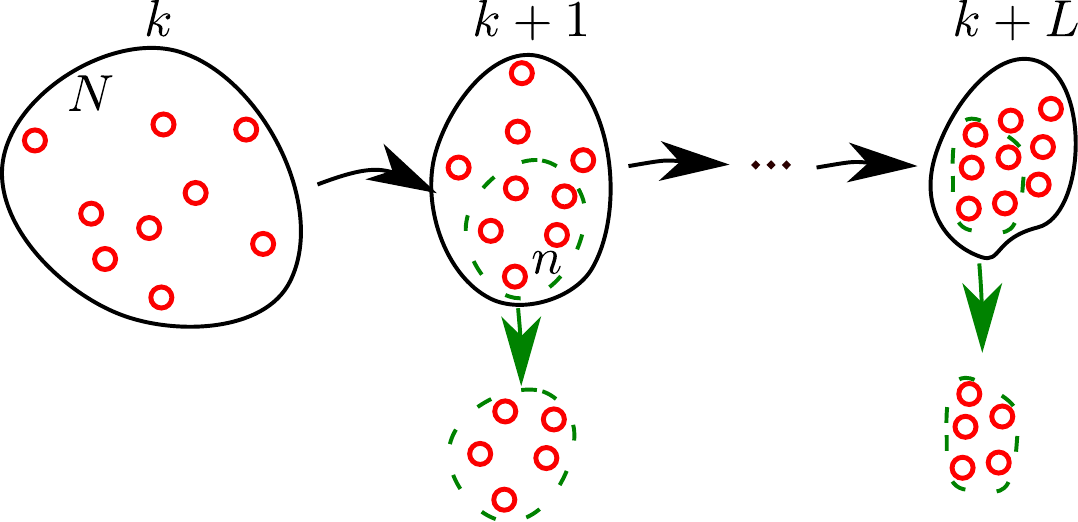}
		\subcaption{}
	\end{minipage} 
	\caption{Potential simplification techniques: \textbf{(a)} Choosing a subset of  samples only at time $k$; \textbf{(b)} Choosing a subset of samples at each time $\ell$.} \label{fig: clouds}
\end{figure}


\section{Specific Simplification} \label{sec: spec}
\subsection{Simplification Technique}

In this section, we exemplify our technique on a specific simplification method. Assume the belief $b_k$ is represented by a set of $N$ weighted samples $\{(w^i_k, x^i_k)\}^N_{i=1}$. Our simplification operator $\nu$ provides a way to choose a subset of $n$ samples from the original $N$ samples. For example, subsampling according to weights. We denote by (a) simplification affecting $b_k$  and producing $\breve{b}_k = \{(w^j_k, x^j_k)\}^n_{j=1}$ ($\nu$ is only applied at time $k$ with fixed seed) \eqref{eq:ConditionalJointTimeReward} and by (b) simplification affecting $b_{\ell}$ and producing $\breve{b}_{\ell} = \{(w^j_{\ell}, x^j_{\ell})\}^n_{j=1}$ for $\ell\in [k+1,k+L]$ \eqref{eq:ConditionalJointTimeReward2}. We take $\psi_{st}$ as an off-the-shelf particle filter, which produces the same number of samples as the input.  The reward operator is approximated by $N$ and $n$ samples in the original and simplified setting, respectively. To make a clear connection to our general framework, in this section, we denote $\rew_{\ell} = r^N_{\ell}$ and  $\rewsimpl_{\ell} = r^n_{\ell}$. The two simplification methods are illustrated in Fig. ~\ref{fig: clouds}.

\subsection{Online Bounds on Sample Based Reward}
To present development for \eqref{eq:LenientConstraint1}, we take inspiration from confidence intervals \cite{wasserman2013all}. Let us introduce the following model
\begin{align}
	\begin{pmatrix} \ret \\  \retsimpl \end{pmatrix} \sep \his_{k+L}, \nu \sim \mathcal{N} 
	\left( \begin{pmatrix} \mu \\ \mu \end{pmatrix} 
	\begin{pmatrix}
		\text{se}^2(N) & \text{cov} \\ \text{cov} & \text{se}^2(n)
	\end{pmatrix} \right), \label{eq:Model}
\end{align}
where $\text{se}$ is the standard error and cov is the covariance. Online we do not have access to these quantities. The standard error depends on the number of samples $N$ and $n$ respectively, dwindling as the number of samples increases.
We assume that each marginal is distributed around the same mean value $\mu$.  
Denote $y=\ret-\retsimpl$. It is known that $y$  is a zero mean  Gaussian with the following variance 
\begin{align}
	\text{var}(y) = \text{se}^2(N) + \text{se}^2(n) -2\text{cov}.
\end{align}
The derivation is in Appendix~\ref{app:BoundsDer}.  
Let ${z = \frac{y}{\sqrt{\text{var}(y)}} \sim \mathcal{N}(0, 1)}$ and
\begin{align}
	z_{\alpha/2} = \Phi^{-1}\left(1-\alpha/2\right),
\end{align}
where $\Phi$ is a Cumulative Distribution Function (CDF) of a standard normal variable  \cite{wasserman2013all} so  $ \probd (z > z_{\alpha/2}) = \alpha/2$ and 
\begin{align}
	\prob \left(-z_{\alpha/2} \leq z \leq z_{\alpha/2} \right)=1-\alpha.
\end{align}
In other words 
\begin{align}
	&\prob \left( |y|  \leq  z_{\alpha/2} \sqrt{\text{var}(y)} \sep \his_{k+L}, \nu\right)= 1-\alpha. \label{eq:ActualConstraint}
\end{align}
Using the facts $\text{se}(N) \leq \text{se}(n)$  and $\text{cov}\leq \text{se}(N)\text{se}(n)$ we obtain that in the case (a) of the simplification ($\text{cov}=0$)
\begin{align}
	\text{var}(y)= \text{se}^2(N) + \text{se}^2(n) \leq 2\text{se}^2(n).
\end{align}
Let us elaborate on why the assumption that $\bsimpl_{k}$ is given alongside $b_k$ and $\nu$ (fixing the seed of subsampling operator) nullifies the covariance between the returns. According to \eqref{eq:ConditionalJointTimeReward} the only source of correlation between returns is \eqref{eq:Stochnu} at time $k$. By fixing the seed, we made \eqref{eq:Stochnu} a  Dirac function. Therefore, conditioning on $b_k$ and $\nu$ is equivalent to conditioning on $\bsimpl_{k}$. 

In  case  of  simplification (b), 
\begin{align}
	\text{var}(y)= \text{se}^2(N) + \text{se}^2(n)  -2\text{cov}\leq 4\text{se}^2(n).
\end{align}
Thus, from \eqref{eq:ActualConstraint} we obtain for both simplification possibilities 
\begin{align}
 &\text{(a) }  \prob \left(|\ret - \retsimpl|  \leq z_{\alpha/2} \sqrt{2}\text{se}(n) \sep \his_{k+L}, \nu\right) \geq 1-\alpha. \label{eq:Simpl_a} \\
 &\text{(b) }  \prob \left( |\ret  -  \retsimpl |  \leq z_{\alpha/2} 2\text{se}(n) \sep \his_{k+L}, \nu \right) \geq 1-\alpha.\label{eq:Simpl_b}
\end{align}

\subsection{Estimation of the Variance}\label{sec:EstVar}
As we do not have access to $\text{se}(n)$ in \eqref{eq:Simpl_a} and \eqref{eq:Simpl_b}, it has to be estimated. The simplest way to do that is to repeatedly sample simplified returns $m$ times from \eqref{eq:ConditionalJointTimeReward} in case of simplification (a) or from \eqref{eq:ConditionalJointTimeReward2} in case of simplification (b).   
Note that a possible bias of the particle filter and the estimation of standard error make  \eqref{eq:ActualConstraint} only asymptotically correct. However, when dealing with a sufficient amount of samples $N$ and $n$, these deviations from \eqref{eq:Model} are negligible. Even with repeated re-sampling we will reduce computational complexity, as we analyze in Section \ref{sec: sim}

Moreover, since we recalculate the simplified reward $m$ times, we could improve the final simplified return. However, this is out of the scope of this paper.    
To conclude, the bounds for both simplification methods are
\begin{align} 
	&\text{(a) } \ub= \retsimpl + z_{\alpha/2} \sqrt{2}\hat{\text{se}}_m \quad \lb= \retsimpl  - z_{\alpha/2} \sqrt{2}\hat{\text{se}}_m  \label{eq:ActualBoundsa}\\
	&\text{(b) } \ub= \retsimpl  + z_{\alpha/2} 2\hat{\text{se}}_m \quad \ \ \lb= \retsimpl  - z_{\alpha/2} 2\hat{\text{se}}_m.
\end{align}	
These bounds asymptotically hold with probability at least $1-\alpha$.  
\subsection{Implementation Details and Computational Complexity}
Now we describe steps in calculating \pbloss. First, we need to construct two extended belief policy trees appropriate to the two candidate policies (see Fig.~\ref{fig:ExtendedBeliefTree}). Second, we shall apply the simplification and calculate simplified returns and bounds.  From now let us assume that  $\psi_{st}$ has a low-variance re-sampler \cite{thrun2005probabilistic}. The entire belief update process complexity is $\mathcal{O}(N)$. Since the extended belief tree does not undergo simplification, it is common to the original and simplified problems. Therefore, we discuss it in the Appendix~\ref{app:TreesCompl}. 
 
Now we analyze the speedup in running time as a result of simplification.  As a momentary reward, we take the differential entropy estimator from 
\cite{skoglar2009information}, \cite{boers2010particle}. 
This selection makes the complexity of calculating the momentary reward to be $\mathcal{O}(N^2)$. Note it is customary to choose resampled or weighted belief for reward calculation. The resampled belief has identical weights. However, the effects of this choice are negligible.

For the bounds calculation in case of simplification (a) we need to apply particle filter with $n$ samples \eqref{eq:ConditionalJointTimeReward} and in case of simplification (b) with $N$ \eqref{eq:ConditionalJointTimeReward2})  $L$ times for each return. Since its complexity is linear in the number of samples, the expected speedup is
\begin{align}
	\frac{N^2}{n^2\cdot m}. \label{eq:speedup}
\end{align}
We obtained this speedup in all our simulations.

\section{Results}
\subsection{Illustration via a Toy Example}
We start with  a toy example. Consider an underwater robot going as low as possible into the sea. 
Let robot's state  at time instant $k$  be $x_{k} \in \mathbb{R}$ (altitude under the sea level). The theoretical distribution of $x_k$ is out of  reach. However, it is not Gaussian.  Assume that the robot's belief about its state $x_k$ at time instant $k$ is represented by a set of $N$ $i.i.d$ samples, i.e.,~$b_k= \{x_{k}^i\}^N_{i=1}$, produced by a perfect particle filter characterized by  $\psi_{st}$. Further, consider the theoretical reward function $\rho_{dt}(b)=\mathbb{E}_{x\sim b}[x]$. Maximizing this reward will take the robot as deep into the sea as possible. The exact calculation of this reward is out of reach due to belief's sampled-based representation. 

Therefore, the theoretical reward is replaced by sample mean, i.e.~$\rew_{k+1} = \frac{1}{N}\sum_{i=1}^N x_{k+1}^i$.  

Suppose the robot considers, as simplification, to use only $n< N$ samples, such that 
$\rewsimpl_{k+1}  = \frac{1}{n}\sum_{j=1}^n x_{k+1}^j$. This simplification could be used in two flavors. The first is to subsample $b_k$. The simplified version, $\breve{b}_k$, serves as an input to a particle filter, producing $\breve{b}_{k+1}$ \eqref{eq:ConditionalJointTimeReward}. The second is to subsample $b_{k+1}$ to produce $\breve{b}_{k+1}$ \eqref{eq:ConditionalJointTimeReward2}. In general, $\breve{b}_{k+1}$ would differ in these two cases. 
 As in previous sections, we denote these two genuine different simplification techniques by $(a)$ and $(b)$. We recall that in the case of $(a)$, we assume a fixed seed of the sampler.  

As an example for stochastic bounds that hold with probability one, one could consider the following analytical bounds. 
\begin{align}
	\min_i \{x^i_{k+1}\} \leq \frac{1}{N}\sum_{i=1}^N x_{k+1}^i \leq \max_i \{x^i_{k+1}\}. \label{eq: toy_stochbounds_with_prob1}
\end{align}
We can define lower and upper bounds as $\lb=\min_i \{x^i_{k+1}\}$ and $\ub=\max_i \{x^i_{k+1}\}$. These bounds bound any sample of the original reward $\frac{1}{N}\sum_{i=1}^N x_{k+1}^i$ constructed from a corresponding set of $N$ state samples. In other words, this is an example of bounds that hold with probability one. Of course, these bounds prone to be affected by outliers (one outlier can take $\ub$ extremely far), and far better analytical bounds could be developed. 

Another possibility is to consider the structure of \eqref{eq:RewSimplRewCond} to define more lenient bounds. From the central limit theorem \cite{durrett2019probability},\cite{wasserman2013all} the following holds asymptotically  
\begin{align}
	\rew_{k+1} \sep  \his_{k+1} \! \sim \!  \mathcal{N}\!\left(\!  \mu,\frac{\sigma^2}{N} \!  \right)\!, \ \! \rewsimpl_{k+1} \sep  \his_{k+1}, \nu \!  \sim \!  \mathcal{N}\left(\! \mu,\frac{\sigma^2}{n}\!  \right). \label{eq:ToyCondMarginals}
\end{align}
In the case of simplification $(a)$, all samples besides $b_k$ and $\breve{b}_{k}$ are independent (fixed seed). Therefore \eqref{eq:RewSimplRewCond} is a multiplication of  marginals. We arrive to \eqref{eq:Simpl_a} with $\text{se}(n) = \frac{\sigma^2}{n}$ (similarly  $\text{se}(N) = \frac{\sigma^2}{N}$). Simplification $(b)$ produces \eqref{eq:Simpl_b}. Note that $\sigma$ is unknown and has to be estimated, as  in Section \ref{sec:EstVar}.

\subsection{Autonomous Navigation with Light Beacons} \label{sec: sim}

\begin{figure}[t] 
	\centering
	{\includegraphics[width=\columnwidth]{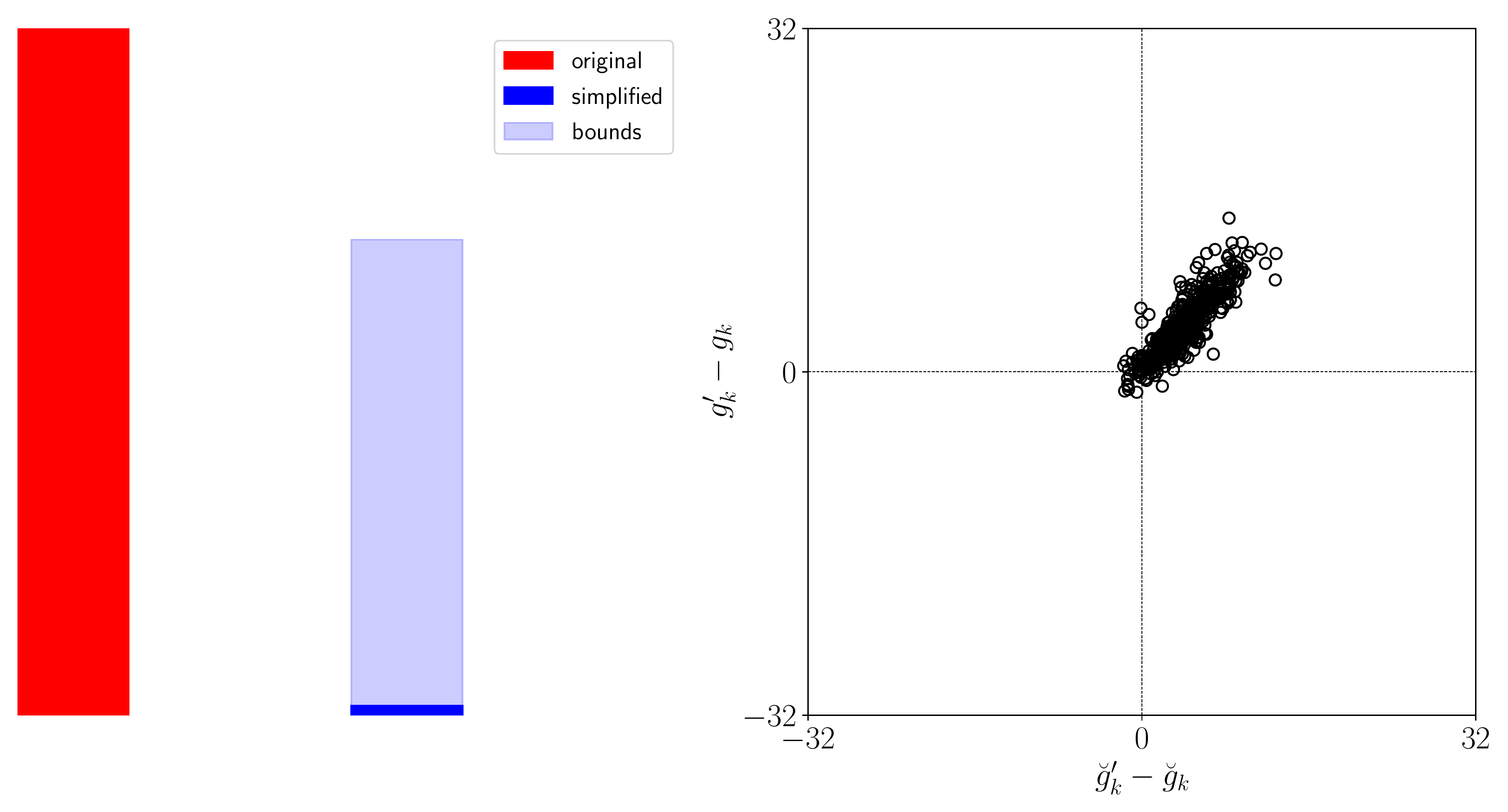}}
	\caption{Results for scenario $1$ - probabilistic action consistency: (left) Demonstration of  runtimes of the total number of the returns for a given extended belief tree where $N=1500$ and $n=175$. Note that this illustration agrees with \eqref{eq:speedup};  (right) action consistency of the samples of the return. }\label{fig: sim_runtimes}
	\label{fig:RunTimesAndDiffs}
\end{figure}
We exemplify our method on the problem of autonomous navigation to a  goal with light beacons, which can be used for localization. In all our simulations in this section, the return $\ret$ is a cumulative reward, and, as a representative projection operator $\varphi$, we chose the expected value. In this study, the simplification is of type (a).
  
For simplicity, assume we have a linear motion model $T$, where $x \in \mathbb{R}^2$ as well as $a \in \mathbb{R}^2$ 
\begin{align}
	x_{k+1} = x_{k} + a_k+w_k \quad w_k \sim \mathcal{N}(0, \Sigma_w),
\end{align}
where $\Sigma_w = w \cdot I$ ($w$ is a given parameter) and action $a_k \in \mathcal{A}$, and where the action space  $\mathcal{A}$ is the space of motion primitives. 
\paragraph{Characterizing Probabilistic Action Consistency}
The observation model $O$ is as follows,  $z \sim \mathcal{N}(x, \Sigma_v(x))$, where the spatially-varying covariance matrix is 
\begin{align}
\Sigma_v(x)= v(x) \cdot I, \quad 	v(x) = w \cdot \min \{1, \|x - x^*\|^2_2\},
\end{align}
where $x^*$ is the location of the light beacon closest to $x$. The noise has a constant variance $w$. Without losing generality, we assume $b_k$ at planning time is uniformly distributed in  a unit square. We set $L=12$ and compare two action sequences: $a_{k+1:k+12}$ is six times $(1, 0)^T$ and after that six times $(0, 1)^T$. In the action sequence $a'_{k+1:k+12}$ we switched the order of actions such that the robot performs six times $(0, 1)^T$ and after that six times $(1, 0)^T$. 

\begin{figure}[t] 
	\centering
	{\includegraphics[width=\columnwidth]{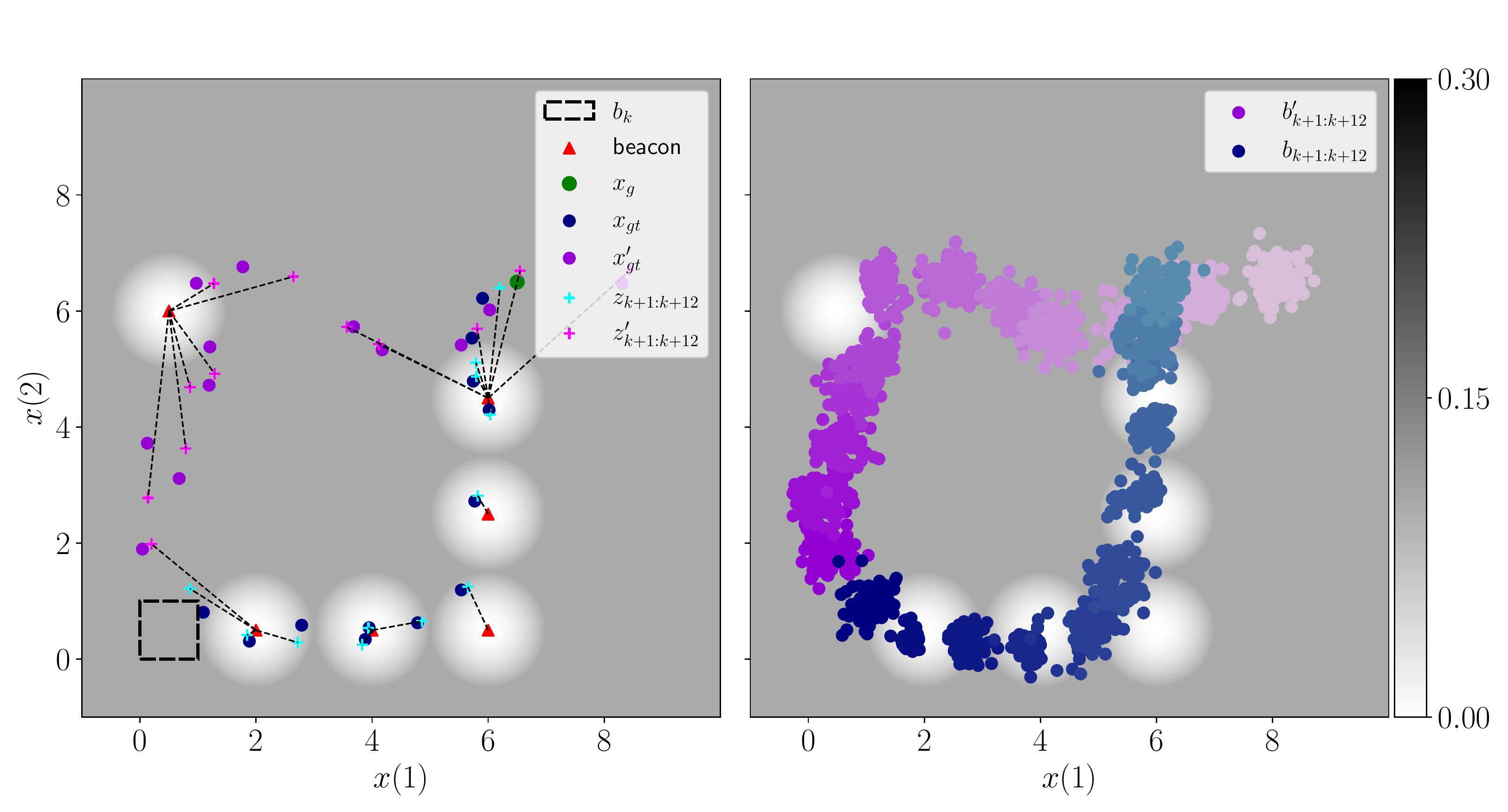}}
	\caption{Results for scenario $1$ - probabilistic action consistency: Illustration of one realization of the future in a simulated scenario considering two possible action sequences. We start from $b_k$ represented by samples uniformly distributed on a unit square. We demonstrated two sequences of observations alongside ground truth state samples, and the closest beacons produced these observations from the left. From the right, we plotted two sequences of the beliefs produced by these two histories. We show $100$ most probable samples of each belief.} \label{fig:SimScen}
\end{figure}
One realization of a possible future in terms of measurements and corresponding posterior beliefs is illustrated in  Fig.~\ref{fig:SimScen}. It is clearly seen that proximity to a beacon improves localization. Note the robot is always able to avoid a dead reckoning scenario as it always gets an observation from the closest beacon. This corresponds to a non-Gaussian noise, promoting the usage of particles-based belief representation.  We hope that this setting conveys a real world scenario where an ambulating robot is  equipped with long and short range sensors. The close range sensors are activated when the robot is inside a unit circle around the beacon. When the robot is outside a unit circle from the closest beacon, the beacon is detectable only by the long range sensors, which are less sensitive.

\begin{table}[t]
	\centering
	\setlength{\extrarowheight}{.5ex}
\begin{tabular}{|c|c|c|c|c|c|}
	\hline 
	$  n $  & \scalebox{0.8}{$  \prob (\mathcal{L} >  0.5 \breve{\Delta}^* \sep \cdot)$}  & $  \beta (0.5 \breve{\Delta}^*)$ & $  \breve{\Delta}^*$  & \scalebox{0.8}{$  \prob(\mathcal{L} > \breve{\Delta}^* \sep \cdot )$}  & $  \beta(\breve{\Delta}^* )$ \\
	\hline 
	$175$ & 0.0 & 0.33 & 4.14 & 0.0 & 0.11 \\
	\hline 
	$150$ & 0.01 & 0.43 & 4.04 & 0.0 & 0.17 \\
	\hline 
	$125$ & 0.01 & 0.43 & 4.21 & 0.0 & 0.2 \\
	\hline 
	$100$ & 0.0 & 0.56 & 4.08 & 0.0 & 0.29 \\
	\hline 
	$75$ & 0.01 & 0.64 & 4.01 & 0.0 & 0.39 \\
	\hline 
	$50$ & 0.02 & 0.83 & 3.72 & 0.01 & 0.63 \\
	\hline 
	$25$ & 0.07 & 1.0 & 3.34 & 0.03 & 0.94 \\
	\hline 
\end{tabular}
	\caption{Results for scenario $1$ - probabilistic action consistency: Online characterization for $N=1500$, $\alpha = 0.01$, $z_{\alpha/2}=2.56$.}
	\label{tbl:suboptimality}
\end{table} 

We present results of simplification $(a)$ for $w=0.1$, $N=1500$, $m=50$, $\alpha = 0.01$, $z_{\alpha/2}=2.56$, and the total number of observations is $500$. For each sample of $z_{k+1:k+L}$, we sampled $b_{k+1:k+L}$ once.   

As we see in the left part of Fig.~\ref{fig: sim_runtimes} we gained speedup as expected \eqref{eq:speedup} for $n=175$. We added measurements of all running times in our simulations to Appendix~\ref{app:addRes}.  

From these samples of the returns and bounds, we build \ploss and \pbloss in Fig.~\ref{fig:plosspblosshist}. In the right part of Fig.~\ref{fig: sim_runtimes} quadrants \Romannum{2} and \Romannum{4}, we observe samples that are not action consistent. To assess performance we need to choose some representative $\Delta$.
Since online we have access exclusively to the simplified problem, let us choose $\breve{\Delta}^* = \sep \mathbb{E}[\retsimpl \sep b_k, \pi, \nu]- \mathbb{E}[\retsimpl' \sep b_k, \pi', \nu]|$ and $\Delta = 0.5\breve{\Delta}^*$.  Table~\ref{tbl:suboptimality} quantifies online characterization against offline \ploss TDF.  

In Fig.~\ref{fig:characteristics} we focused on $n=175$;  \emph{online} we can conclude that probability that loss incurred by this simplification will be greater than $\breve{\Delta}^*$  is at most $0.11$, while actual $ \prob(\mathcal{L} > \breve{\Delta}^* \sep \cdot )$ is $0.0$. Similarly, the probability for loss incurred by this simplification to be greater than $0.5\breve{\Delta}^*$  is at most $0.33$, while actual $ \prob(\mathcal{L} > 0.5\breve{\Delta}^* \sep \cdot )$ is $0.0$. In this scenario, the simplification is not absolute action consistent; it means variability described by  \eqref{eq:Model} is sufficient to switch the order of the returns and incur loss $\Delta$ at some sampled realization.  

Furthermore, our bounds depend on variance ($\text{se}^2(n)$) of the sample approximation of the reward \eqref{eq:ActualBoundsa}, which, according to \eqref{eq:Model} does not depend on $\Delta$. Hence, as $\Delta$ decreases towards zero, the contribution of variance versus the difference between simplified returns grows for  any realization of $\bar{\mathcal{L}}$.  
Therefore, \pbloss departs from \ploss as $\Delta$ decreases. We observe this behavior in Fig.~\ref{fig:characteristics}. Moreover, with the diminishing number of samples, this effect is amplified, as demonstrated in Fig.~\ref{fig:Scenario1Study}, due to growing variance \eqref{eq:Model}.  Remarkably,  when samples of original returns are more distinct, the effect of variance is nullified. In such a setting, our characterization is incredibly precise, see Fig.~\ref{fig:Scenario2Study}. 

Thus, the behavior of the \pbloss is more conservative in more delicate scenarios, where  two candidate policies are close to each other in terms of returns. Importantly, for significantly different policies, \pbloss becomes tighter to \ploss.
This brings us to the next section.  

\begin{figure}[t] 
	\centering
	{\includegraphics[width=\columnwidth]{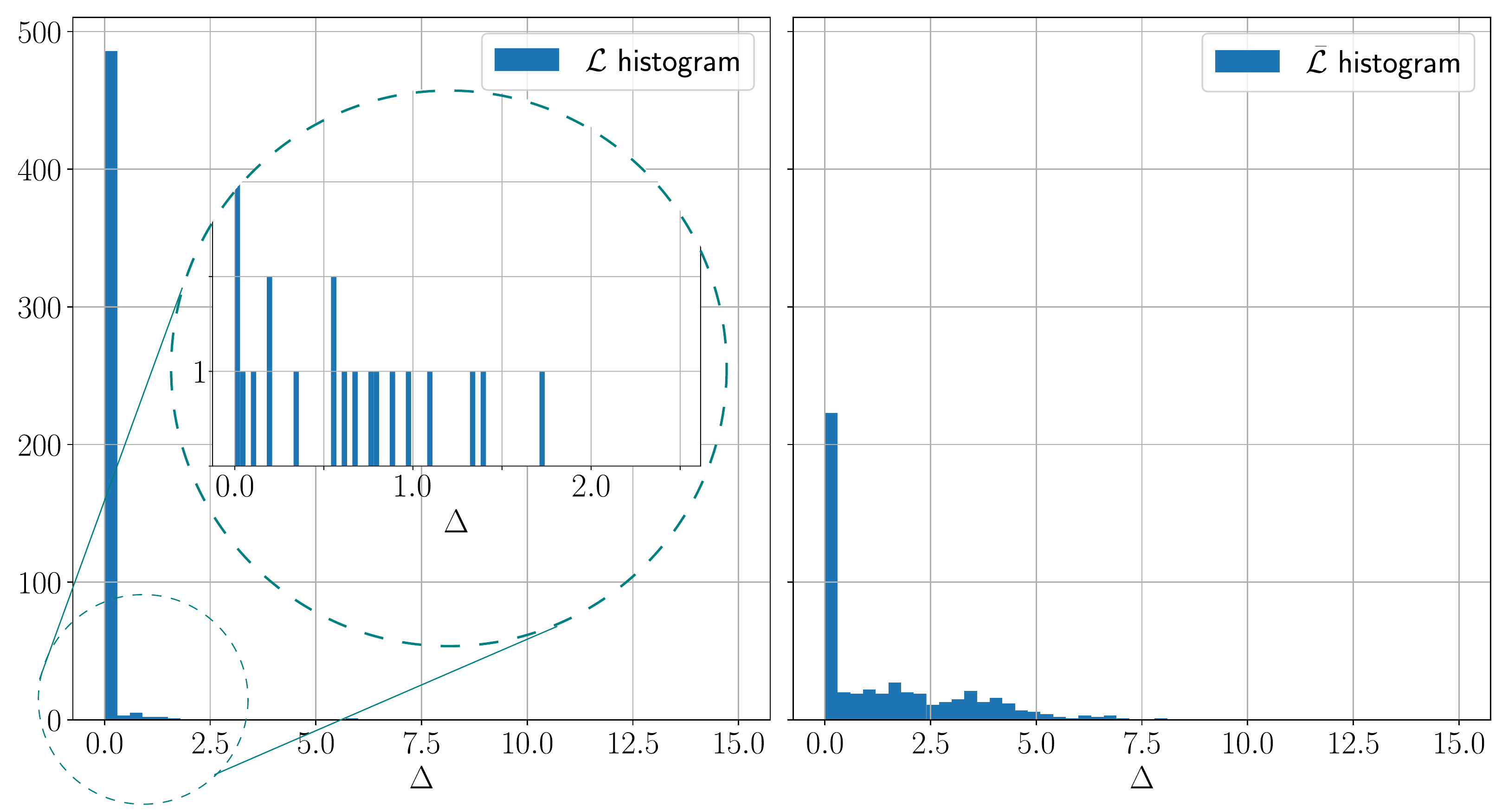}}
	\caption{Results for scenario $1$ - probabilistic action consistency: Histograms of \ploss and \pbloss for for $N=1500$, $n=175$, $\alpha = 0.01$, $z_{\alpha/2}=2.56$  (bin width is $0.3$, in zoom-in, bin width is $0.03$). }\label{fig:plosspblosshist}
\end{figure}

\begin{figure}[t] 
	\centering
	{\includegraphics[width=\columnwidth]{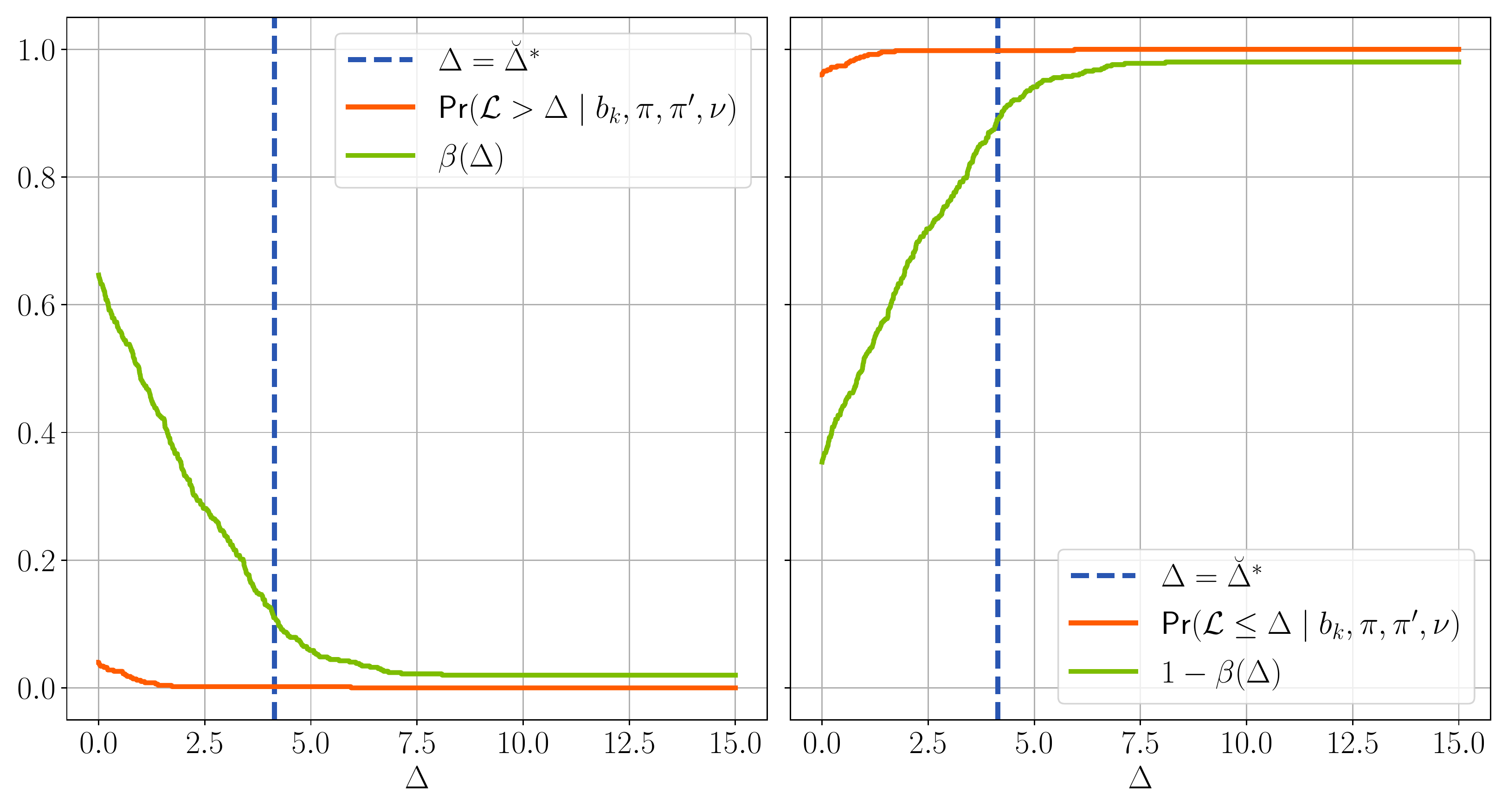}}
	\caption{Results for scenario $1$ - probabilistic action consistency: Empirical characterization for $N=1500$, $n=175$, $\alpha = 0.01$, $z_{\alpha/2}=2.56$, evaluated in a grid with intervals  $0.001$. }  \label{fig:characteristics}
\end{figure}

\paragraph{Revealing Empirical Absolute Action Consistency}
In this scenario we modified the noise in the observation model as such $v(x) = w \cdot  \|x - x^*\|^2_2$. 
In addition we removed one beacon on the way of the second action sequence. We remain with $w=0.1$, $m=50$, $\alpha = 0.01$, $z_{\alpha/2}=2.56$ and set $N=1000$.  In this scenario the  returns of two action sequences are much more distant. The samples in the right segment of Fig.~\ref{fig:RunTimesAndDiffsAbs} are more distant from the origin than in Fig.~\ref{fig:RunTimesAndDiffs}.  The characterization is shown in  Table~\ref{tbl:suboptimality_abs}. Therefore, the simplification is empirically absolute action consistent. As we see from the Table~\ref{tbl:suboptimality_abs}, observing $\beta(\Delta=0.0)$ we are able to identify online that for $n=100$ and $n=75$, probability to receive samples of the returns violating action consistency is at most $0.03$, while  $  \prob(\mathcal{L} > 0.0 \sep \cdot)$  is $0.0$.   
\begin{figure}[t] 
	\centering
	{\includegraphics[width=0.5\textwidth]{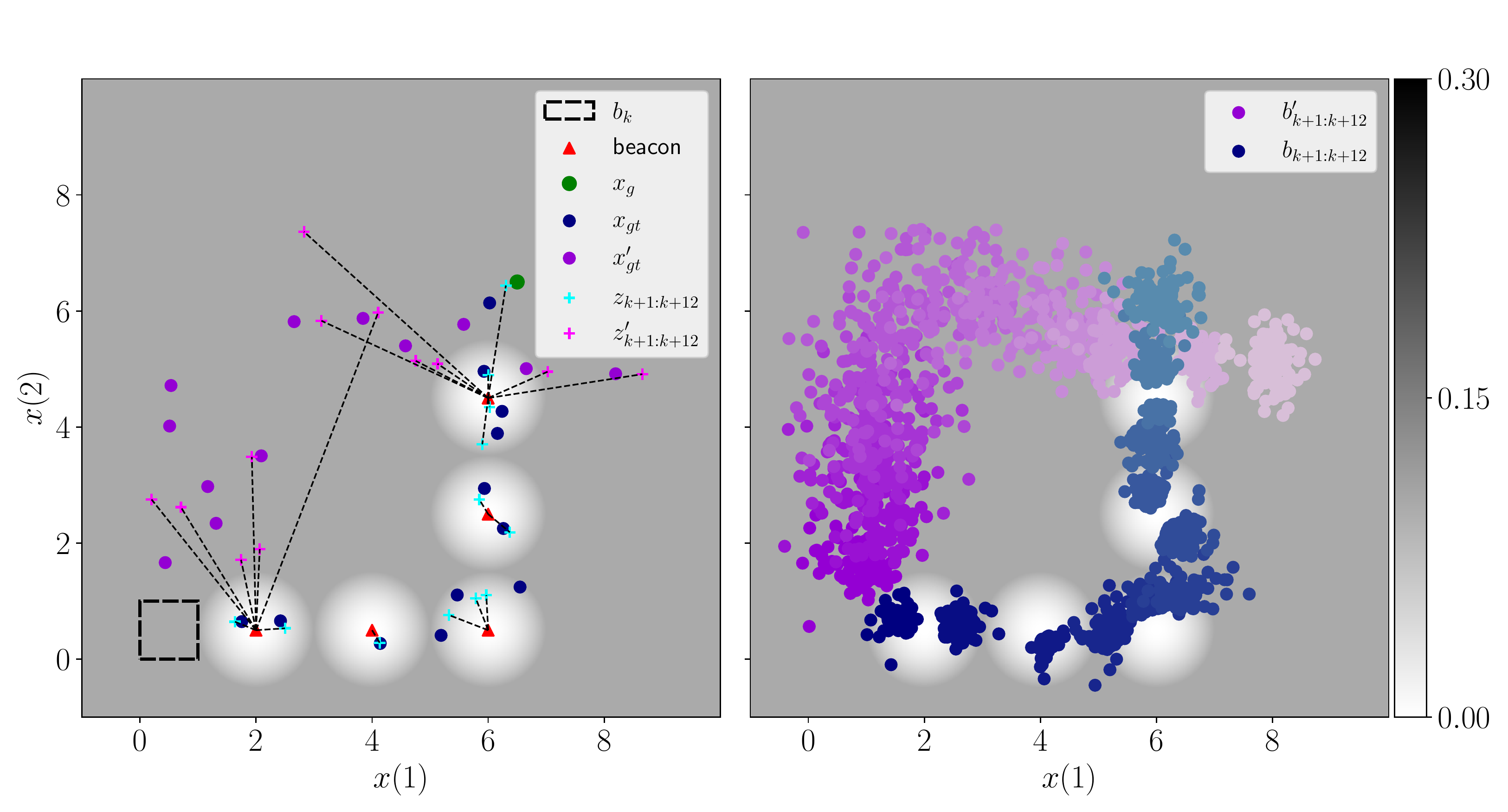}}
	\caption{Results for scenario $2$ - empirical absolute action consistency: Illustration of one realization of the future in a simulated scenario considering two possible action sequences. We start from $b_k$ represented by samples uniformly distributed on a unit square. We demonstrated two sequences of observations alongside ground truth state samples, and the closest beacons produced these observations from the left. From the right, we plotted two sequences of the beliefs produced by these two histories. We show $100$ most probable samples of each belief.} \label{fig:SimScen2}
\end{figure}

\begin{table}[t]
	\centering
	\setlength{\extrarowheight}{.5ex}
		\begin{tabular}{|c|c|c|c|c|c|}
	\hline 
	$  n $  & \scalebox{0.8}{$  \prob(\mathcal{L} > 0.0 \sep \cdot)$}  & $  \beta (\Delta=0.0)$ & $  \breve{\Delta}^*$  & \scalebox{0.8}{$  \prob(\mathcal{L} > \breve{\Delta}^* \sep \cdot )$ } & $  \beta(\breve{\Delta}^* )$ \\
	\hline 
	$100$ & 0.0 & 0.03 & 17.54 & 0.0 & 0.02 \\
	\hline 
	$75$ & 0.0 & 0.03 & 17.14 & 0.0 & 0.02 \\
	\hline 
	$50$ & 0.0 & 0.06 & 16.65 & 0.0 & 0.02 \\
	\hline 
	$25$ & 0.0 & 0.19 & 15.27 & 0.0 & 0.02 \\
	\hline 
\end{tabular}
	\caption{Results for scenario $2$ - empirical absolute action consistency: Online characterization for $N=1000$, $\alpha = 0.01$, $z_{\alpha/2}=2.56$.}
	\label{tbl:suboptimality_abs}
\end{table} 
\begin{figure}[t] 
	\centering
	{\includegraphics[width=0.5\textwidth]{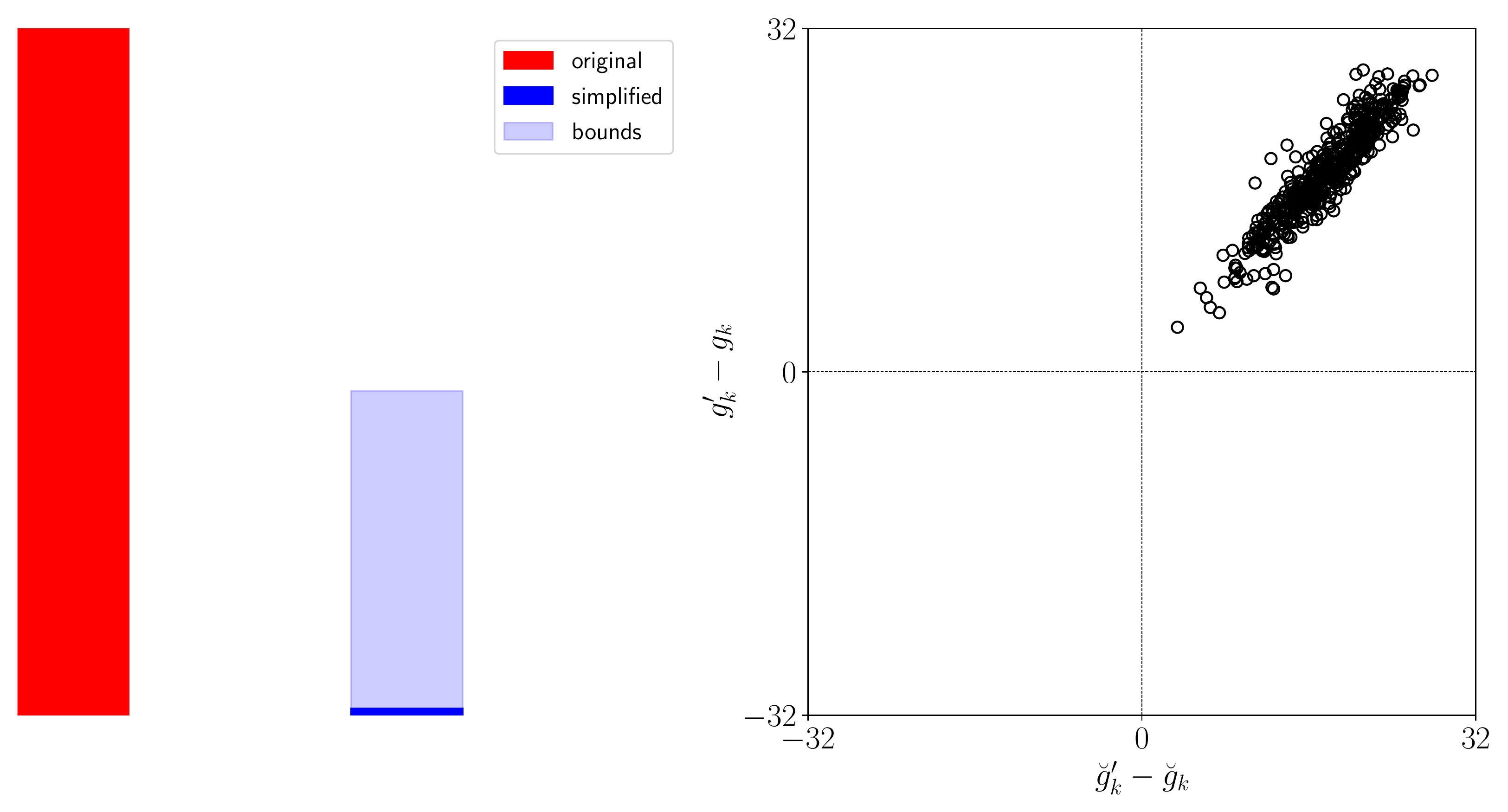}}
	\caption{Results for scenario $2$ - empirical absolute action consistency: (left) Demonstration of  runtimes of the total number of the returns for a given extended belief tree where $N=1000$ and $n=100$. Note that this illustration agrees with \eqref{eq:speedup};  (right) action consistency of the samples of the return.} 
	\label{fig:RunTimesAndDiffsAbs}
\end{figure}
We placed additional results and discussions in Appendix~\ref{app:addRes}.   

\section{Conclusion} 
\label{sec:conclusion}
We extended $\rho$-POMDP to \extendedpomdp and introduced a framework for quantifying online the effect of simplification, alongside novel stochastic bounds on the return. Our bounds take advantage of the information encoded in the joint distribution of the original and simplified return. The proposed general framework is applicable to any bounds on the return to capture simplification outcomes. We presented experiments that confirmed the benefits of our approach. Our future research will focus on incorporating the contributed framework within simplified online decision making with probabilistic guarantees.


\section*{Acknowledgments}

This research was  supported by the Israel Science Foundation (ISF),  by the Israel Ministry of Science \& Technology (MOST), and by 
a donation from the Zuckerman Fund to the Technion Center for Machine Learning and Intelligent Systems (MLIS).


\bibliographystyle{plainnat}
\bibliography{references}

\begin{thebibliography}{37}
\providecommand{\natexlab}[1]{#1}
\providecommand{\url}[1]{\texttt{#1}}
\expandafter\ifx\csname urlstyle\endcsname\relax
  \providecommand{\doi}[1]{doi: #1}\else
  \providecommand{\doi}{doi: \begingroup \urlstyle{rm}\Url}\fi

\bibitem[Araya et~al.(2010)Araya, Buffet, Thomas, and
  Charpillet]{araya2010pomdp}
Mauricio Araya, Olivier Buffet, Vincent Thomas, and Fran{\c{c}}cois Charpillet.
\newblock A pomdp extension with belief-dependent rewards.
\newblock \emph{Advances in neural information processing systems},
  23:\penalty0 64--72, 2010.

\bibitem[Bai et~al.(2014)Bai, Hsu, and Lee]{bai2014integrated}
Haoyu Bai, David Hsu, and Wee~Sun Lee.
\newblock Integrated perception and planning in the continuous space: A pomdp
  approach.
\newblock \emph{The International Journal of Robotics Research}, 33\penalty0
  (9):\penalty0 1288--1302, 2014.

\bibitem[Bertsekas(2017)]{bertsekas2017dynamic}
Dimitri~P. Bertsekas.
\newblock \emph{Dynamic programming and optimal control. 4ed}, volume~1.
\newblock Athena scientific Belmont, MA, 2017.

\bibitem[Boers et~al.(2010)Boers, Driessen, Bagchi, and
  Mandal]{boers2010particle}
Yvo Boers, Hans Driessen, Arunabha Bagchi, and Pranab Mandal.
\newblock Particle filter based entropy.
\newblock In \emph{2010 13th International Conference on Information Fusion},
  pages 1--8. IEEE, 2010.

\bibitem[Bouakiz and Kebir(1995)]{bouakiz1995target}
Mokrane Bouakiz and Youcef Kebir.
\newblock Target-level criterion in markov decision processes.
\newblock \emph{Journal of Optimization Theory and Applications}, 86\penalty0
  (1):\penalty0 1--15, 1995.

\bibitem[Defourny et~al.(2008)Defourny, Ernst, and Wehenkel]{Defourny08ws}
Boris Defourny, Damien Ernst, and Louis Wehenkel.
\newblock Risk-aware decision making and dynamic programming.
\newblock In \emph{NIPS Workshop on Model Uncertainty and Risk in RL}, 2008.

\bibitem[Dressel and Kochenderfer(2017)]{dressel2017efficient}
Louis Dressel and Mykel~J Kochenderfer.
\newblock Efficient decision-theoretic target localization.
\newblock In \emph{Twenty-Seventh International Conference on Automated
  Planning and Scheduling}, 2017.

\bibitem[Durrett(2019)]{durrett2019probability}
Rick Durrett.
\newblock \emph{Probability: theory and examples}, volume~49.
\newblock Cambridge university press, 2019.

\bibitem[Elimelech and Indelman(2019)]{elimelech2019efficient}
Khen Elimelech and Vadim Indelman.
\newblock Simplified decision making in the belief space using belief
  sparsification.
\newblock \emph{arXiv preprint arXiv:1909.00885}, 2019.

\bibitem[Elimelech and Indelman(2020)]{elimelech2020fast}
Khen Elimelech and Vadim Indelman.
\newblock Fast action elimination for efficient decision making and belief
  space planning using bounded approximations.
\newblock In \emph{Robotics Research}, pages 843--858. Springer, 2020.

\bibitem[Farhi and Indelman(2019)]{farhi2019ix}
Elad~I Farhi and Vadim Indelman.
\newblock ix-bsp: Belief space planning through incremental expectation.
\newblock In \emph{2019 International Conference on Robotics and Automation
  (ICRA)}, pages 7180--7186. IEEE, 2019.

\bibitem[Farhi and Indelman(2021)]{farhi2021ixBSP}
Elad~I. Farhi and Vadim Indelman.
\newblock ix-bsp: Incremental belief space planning.
\newblock \emph{arXiv preprint arXiv:2102.09539}, 2021.

\bibitem[Fehr et~al.(2018)Fehr, Buffet, Thomas, and Dibangoye]{fehr2018rho}
Mathieu Fehr, Olivier Buffet, Vincent Thomas, and Jilles Dibangoye.
\newblock rho-pomdps have lipschitz-continuous epsilon-optimal value functions.
\newblock \emph{Advances in neural information processing systems},
  31:\penalty0 6933--6943, 2018.

\bibitem[Fischer and Tas(2020)]{fischer2020information}
Johannes Fischer and {\"O}mer~Sahin Tas.
\newblock Information particle filter tree: An online algorithm for pomdps with
  belief-based rewards on continuous domains.
\newblock In \emph{International Conference on Machine Learning}, pages
  3177--3187. PMLR, 2020.

\bibitem[Garg et~al.(2019)Garg, Hsu, and Lee]{garg2019despot}
Neha~Priyadarshini Garg, David Hsu, and Wee~Sun Lee.
\newblock Despot-alpha: Online pomdp planning with large state and observation
  spaces.
\newblock In \emph{Robotics: Science and Systems}, 2019.

\bibitem[Indelman(2016)]{indelman2016no}
Vadim Indelman.
\newblock No correlations involved: Decision making under uncertainty in a
  conservative sparse information space.
\newblock \emph{IEEE Robotics and Automation Letters}, 1\penalty0 (1):\penalty0
  407--414, 2016.

\bibitem[Indelman et~al.(2015)Indelman, Carlone, and
  Dellaert]{indelman2015planning}
Vadim Indelman, Luca Carlone, and Frank Dellaert.
\newblock Planning in the continuous domain: A generalized belief space
  approach for autonomous navigation in unknown environments.
\newblock \emph{The International Journal of Robotics Research}, 34\penalty0
  (7):\penalty0 849--882, 2015.

\bibitem[Kaelbling et~al.(1998)Kaelbling, Littman, and
  Cassandra]{Kaelbling98ai}
L.~P. Kaelbling, M.~L. Littman, and A.~R. Cassandra.
\newblock Planning and acting in partially observable stochastic domains.
\newblock \emph{Artificial intelligence}, 101\penalty0 (1):\penalty0 99--134,
  1998.

\bibitem[Kitanov and Indelman(2019)]{kitanov2019topological}
Andrej Kitanov and Vadim Indelman.
\newblock Topological information-theoretic belief space planning with
  optimality guarantees.
\newblock \emph{arXiv preprint arXiv:1903.00927}, 2019.

\bibitem[Knuth(2014)]{knuth2014art}
Donald~E Knuth.
\newblock \emph{Art of computer programming, volume 2: Seminumerical
  algorithms}.
\newblock Addison-Wesley Professional, 2014.

\bibitem[Kochenderfer(2015)]{kochenderfer2015decision}
Mykel~J Kochenderfer.
\newblock \emph{Decision making under uncertainty: theory and application}.
\newblock MIT press, 2015.

\bibitem[Kocsis and Szepesv{\'a}ri(2006)]{kocsis2006bandit}
Levente Kocsis and Csaba Szepesv{\'a}ri.
\newblock Bandit based monte-carlo planning.
\newblock In \emph{European conference on machine learning}, pages 282--293.
  Springer, 2006.

\bibitem[Kurniawati and Yadav(2016)]{kurniawati2016online}
Hanna Kurniawati and Vinay Yadav.
\newblock An online pomdp solver for uncertainty planning in dynamic
  environment.
\newblock In \emph{Robotics Research}, pages 611--629. Springer, 2016.

\bibitem[Kurniawati et~al.(2008)Kurniawati, Hsu, and Lee]{kurniawati2008sarsop}
Hanna Kurniawati, David Hsu, and Wee~Sun Lee.
\newblock Sarsop: Efficient point-based pomdp planning by approximating
  optimally reachable belief spaces.
\newblock In \emph{Robotics: Science and systems}, volume 2008. Zurich,
  Switzerland., 2008.

\bibitem[Papadimitriou and Tsitsiklis(1987)]{papadimitriou1987complexity}
Christos~H Papadimitriou and John~N Tsitsiklis.
\newblock The complexity of markov decision processes.
\newblock \emph{Mathematics of operations research}, 12\penalty0 (3):\penalty0
  441--450, 1987.

\bibitem[Rockafellar et~al.(2000)Rockafellar, Uryasev, et~al.]{Rockafellar00jr}
R~Tyrrell Rockafellar, Stanislav Uryasev, et~al.
\newblock Optimization of conditional value-at-risk.
\newblock \emph{Journal of risk}, 2:\penalty0 21--42, 2000.

\bibitem[Ryan(2008)]{ryan2008information}
Allison Ryan.
\newblock Information-theoretic tracking control based on particle filter
  estimate.
\newblock In \emph{AIAA Guidance, Navigation and Control Conference and
  Exhibit}, page 6307, 2008.

\bibitem[Silver and Veness(2010)]{silver2010monte}
David Silver and Joel Veness.
\newblock Monte-carlo planning in large pomdps.
\newblock In \emph{Advances in neural information processing systems}, pages
  2164--2172, 2010.

\bibitem[Skoglar et~al.(2009)Skoglar, Orguner, and
  Gustafsson]{skoglar2009information}
Per Skoglar, Umut Orguner, and Fredrik Gustafsson.
\newblock On information measures based on particle mixture for optimal
  bearings-only tracking.
\newblock In \emph{2009 IEEE Aerospace conference}, pages 1--14. IEEE, 2009.

\bibitem[Somani et~al.(2013)Somani, Ye, Hsu, and Lee]{somani2013despot}
Adhiraj Somani, Nan Ye, David Hsu, and Wee~Sun Lee.
\newblock Despot: Online pomdp planning with regularization.
\newblock \emph{Advances in neural information processing systems},
  26:\penalty0 1772--1780, 2013.

\bibitem[Spaan et~al.(2015)Spaan, Veiga, and Lima]{spaan2015decision}
Matthijs~TJ Spaan, Tiago~S Veiga, and Pedro~U Lima.
\newblock Decision-theoretic planning under uncertainty with information
  rewards for active cooperative perception.
\newblock \emph{Autonomous Agents and Multi-Agent Systems}, 29\penalty0
  (6):\penalty0 1157--1185, 2015.

\bibitem[Sunberg and Kochenderfer(2017)]{sunberg2017online}
Zachary Sunberg and Mykel Kochenderfer.
\newblock Online algorithms for pomdps with continuous state, action, and
  observation spaces.
\newblock \emph{arXiv preprint arXiv:1709.06196}, 2017.

\bibitem[Sutton and Barto(2018)]{sutton2018reinforcement}
Richard~S Sutton and Andrew~G Barto.
\newblock \emph{Reinforcement learning: An introduction}.
\newblock MIT press, 2018.

\bibitem[Thrun et~al.(2005)Thrun, Burgard, and Fox]{thrun2005probabilistic}
Sebastian Thrun, Wolfram Burgard, and Dieter Fox.
\newblock \emph{Probabilistic robotics}.
\newblock MIT press, 2005.

\bibitem[Wasserman(2013)]{wasserman2013all}
Larry Wasserman.
\newblock \emph{All of statistics: a concise course in statistical inference}.
\newblock Springer Science \& Business Media, 2013.

\bibitem[Wu and Lin(1999)]{wu1999minimizing}
Congbin Wu and Yuanlie Lin.
\newblock Minimizing risk models in markov decision processes with policies
  depending on target values.
\newblock \emph{Journal of mathematical analysis and applications},
  231\penalty0 (1):\penalty0 47--67, 1999.

\bibitem[Ye et~al.(2017)Ye, Somani, Hsu, and Lee]{Ye17jair}
Nan Ye, Adhiraj Somani, David Hsu, and Wee~Sun Lee.
\newblock Despot: Online pomdp planning with regularization.
\newblock \emph{JAIR}, 58:\penalty0 231--266, 2017.

\end{thebibliography}
\begin{appendices}
\onecolumn
\newpage

\section{Discussion on Belief-MDP} \label{app:DiscussionBeliefMDP}
In belief-MDP 
\begin{align}
&\probd \left(b_{\ell} \sep b_{\ell-1}, \pi_{\ell-1}(b_{\ell-1})\right)= \int_{z_{\ell}} \probd(b_{\ell} \sep b_{\ell-1}, \pi_{\ell-1}(b_{\ell-1}), z_{\ell})\probd(z_{\ell} \sep b_{\ell-1}, \pi_{\ell-1}(b_{\ell-1})),
\end{align}
where the usual assumption is that $\probd(b_{\ell} \sep b_{\ell-1}, \pi_{\ell-1}(b_{\ell-1}), z_{\ell})$ is Dirac's delta function. On the contrary, in our extended setting, this distribution is arbitrary. 
\section{Proofs}\label{app:Proofs}
\begin{thm}[Probability that bound bounds] 
	If \begin{align}
	(1-\alpha) \leq \prob \left(\left[\lb \leq \ret \leq \ub \right] = 1 \sep   \his_{k+L},\nu\right) 
	\end{align} 
	and 
	\begin{align}
	(1-\alpha)\leq \prob \left(\left[\lb' \leq \ret' \leq \ub'\right]  = 1 \sep   \his'_{k+L}, \nu \right),
	\end{align}	
	so
	\begin{align}
	&\prob \left([\bar{\mathcal{L}} \geq \mathcal{L}]=1 \sep  b_k, \pi, \pi', \nu \right)  \geq (1-\alpha)^2.
	\end{align}
\end{thm}

\begin{proof}
	By definition
	\begin{align}
	&\prob ([\bar{\mathcal{L}} \geq \mathcal{L}]=1  \sep \left[\lb \leq  \ret \leq \ub \right] =1 , \left[\lb' \leq  \ret' \leq \ub' \right] =1 ,  b_k, \pi, \pi', z_{k+}, z'_{k+}, \nu )= 1.
	\end{align}	
	We first apply marginalization over future observations $z_{k+}\equiv z_{k+1:k+L}$  and $z'_{k+}\equiv z'_{k+1:k+L}$, and events $\left[\lb \leq \ret \leq \ub \right]$ and $\left[\lb' \leq \ret' \leq \ub'\right]$. We then use the fact that given two histories $\his_{k+L} \bydef \{b_k, \pi, z_{k+}\}$ and  $\his'_{k+L} \bydef \{b_k, \pi', z'_{k+}\}$, the events $\left[\lb \leq \ret \leq \ub \right]$ and $\left[\lb' \leq \ret' \leq \ub'\right]$ are independent of each other. Furthermore, each such event depends exclusively on its own history by design. We have that
	\begin{align}
	&\prob \left([\bar{\mathcal{L}} \geq \mathcal{L}] =1 \sep  b_k, \pi, \pi', \nu \right) = \int_{{\substack{z_{k+} \\ z'_{k+} }}} \prob \left(\bar{\mathcal{L}} \geq \mathcal{L}  \sep b_k, \pi, \pi', z_{k+}, z'_{k+}, \nu \right) \probd \left(z_{k+}, z'_{k+} \mid b_k, \pi, \pi'\right)  \geq \\
	&\int_{{\substack{z_{k+} \\ z'_{k+} }}} \prob \left([\bar{\mathcal{L}} \geq \mathcal{L}]=1 , \left[\lb \leq \ret\leq \ub \right] =1 , \left[\lb' \leq  \ret' \leq \ub' \right]=1  \sep b_k, \pi, \pi', z_{k+}, z'_{k+}, \nu \right) \probd \left(z_{k+}, z'_{k+} | b_k, \pi, \pi' \right) =\\
	&\int_{{\substack{z_{k+} \\ z'_{k+} }}} \prob \left([\bar{\mathcal{L}} \geq \mathcal{L}]=1 \sep \left[\lb \leq \ret \leq \ub \right] =1 , \left[\lb' \leq  \ret' \leq \ub' \right]=1  , b_k, \pi, \pi', z_{k+}, z'_{k+}, \nu \right) \nonumber\\
	&\prob \left(\left[\lb \leq  \ret \leq \ub \right] =1 , \left[\lb' \leq \ret' \leq \ub' \right]=1  \sep b_k, \pi, \pi', z_{k+}, z'_{k+}, \nu \right) \probd \left(z_{k+}, z'_{k+} \sep b_k, \pi, \pi'\right) =\\
	&\int_{{\substack{z_{k+} \\ z'_{k+} }}}  \prob \left(\left[\lb \leq \ret\leq \ub \right] =1  \sep b_k, \pi,  z_{k+}, \nu \right) \prob \left(\left[\lb' \leq \ret' \leq \ub' \right]=1  \mid b_k, \pi', z'_{k+}, \nu \right) \probd \left(z_{k+}, z'_{k+} | b_k, \pi, \pi'\right)  \geq (1-\alpha)^2.
	\end{align}
	This completes the proof.
\end{proof}

\begin{thm} [Conditional TDF Lower bound]
	\begin{align}		
	&\prob \left(\mathcal{L} > \Delta \sep \bar{\mathcal{L}}  \geq \mathcal{L} , b_k , \pi, \pi', \nu \right) \leq  \prob \left(\bar{\mathcal{L}}> \Delta \sep \bar{\mathcal{L}} \geq \mathcal{L} , b_k, \pi, \pi', \nu \right) \quad \forall \Delta \in \mathbb{R}. 
	\end{align}
\end{thm}
\begin{proof}
	Let us recall the definition of the probability space and the random variable \cite{durrett2019probability}. 
	A probability space is a triple
	$(\Omega, \mathcal{F}, \prob)$, where $\Omega$ is a set of "outcomes", $\mathcal{F}$ is a set of "events", and
	$\prob : \mathcal{F} \to [0, 1]$ is a function that assigns probabilities to events.  $\mathcal{F}$ is a $\sigma$-field, i.e., a (nonempty) collection
	of subsets of $\Omega$ that satisfy certain conditions. 
	
	A real valued function $X$ defined on $\Omega$ is
	said to be a random variable if for every Borel set $B \subset \mathbb{R}$ we have
	\begin{align}
	X^{-1}(B) = \{\omega : X(\omega) \in B\} \in \mathcal{F}.
	\end{align}
	In our case,  $\Omega$ is the same for $\mathcal{L}$ and $\bar{\mathcal{L}}$, since we condition on the same information
	\begin{align}
	\prob \left(\mathcal{L} > \Delta \sep \bar{\mathcal{L}} \geq \mathcal{L} ,b_k, \pi, \pi', \nu \right) = \prob \left(\{\omega \subseteq \Omega: \mathcal{L}(\omega) > \Delta \} \right).
	\end{align} 
	It follows that 
	\begin{align}
	\{\omega \subseteq \Omega : \mathcal{L}(\omega) > \Delta \} \subseteq \{\omega \subseteq \Omega : \bar{\mathcal{L}}(\omega) > \Delta \}.
	\end{align}
	This completes the proof.
\end{proof}

\begin{thm} [Upper and Lower bounds]
	Denote $\beta(\Delta) \triangleq \min \left\{1, \frac{\prob \left(\bar{\mathcal{L}} > \Delta \sep b_k, \pi, \pi', \nu \right)}{(1-\alpha)^2} +2\alpha - \alpha^2 \right\}$, so
	\begin{align}
	\prob \left(\mathcal{L} > \Delta \sep b_k, \pi, \pi', \nu\right)  \leq  \beta(\Delta)
	\end{align}
	and 
	\begin{align}
	&\prob \left(\mathcal{L} \leq \Delta \sep b_k, \pi, \pi', \nu\right)  \geq 1 -  \beta(\Delta) .
	\end{align}
\end{thm}
\begin{proof}
	To shorten  notations let us denote  $\sep b_k, \pi, \pi', \nu$  by $\sep \cdot$ in the proof. Let us express \ploss TDF as  
	\begin{align}
	\prob \left(\mathcal{L} > \Delta  \sep \cdot\right)  = \prob \left(\mathcal{L} > \Delta  \sep \bar{\mathcal{L} }\geq \mathcal{L}, \cdot \right) \prob \left(\bar{\mathcal{L}} \geq \mathcal{L} \sep \cdot \right) + \prob \left(\mathcal{L} > \Delta \sep \bar{\mathcal{L}} < \mathcal{L}, \cdot \right) \prob \left(\bar{\mathcal{L}} < \mathcal{L}  \sep \cdot\right).
	\end{align}
	Similarly, \pbloss TDF reads
	\begin{align}
	\prob \left(\bar{\mathcal{L}} > \Delta \sep \cdot \right) = \prob \left(\bar{\mathcal{L}}> \Delta  \sep \bar{\mathcal{L} } \geq \mathcal{L}, \cdot \right) \prob \left(\bar{\mathcal{L}} \geq \mathcal{L} \sep \cdot \right) + \prob \left(\bar{\mathcal{L}}> \Delta  \sep \bar{\mathcal{L}} < \mathcal{L}, \cdot \right) \prob \left(\bar{\mathcal{L}} < \mathcal{L}\sep \cdot \right).
	\end{align}
	Since $\alpha \in [0, 1)$ it exists $c \in \mathbb{R}$ such that  
	\begin{align}
	\prob \left(\bar{\mathcal{L}} > \Delta \sep \cdot \right) = c (1-\alpha)^2.
	\end{align}	 
	This implies  
	\begin{align}
	&\prob \left(\bar{\mathcal{L}}> \Delta \sep \bar{\mathcal{L}} \geq \mathcal{L}, \cdot \right) \prob \left(\bar{\mathcal{L}} \geq \mathcal{L} \sep \cdot \right) \leq c(1-\alpha)^2, \\
	&\prob \left(\bar{\mathcal{L}}> \Delta \sep \bar{\mathcal{L}} \geq \mathcal{L}, \cdot \right) \leq c \underbrace{\frac{(1-\alpha)^2}{\prob \left(\bar{\mathcal{L}} \geq \mathcal{L} \sep \cdot \right)}}_{\leq 1}  \leq c.
	\end{align}
	Moreover, using that $\prob(\bar{\mathcal{L}} \geq \mathcal{L} \sep  \cdot ) + \prob(\bar{\mathcal{L}} < \mathcal{L} \sep  \cdot )=1$, we obtain
	\begin{align}
	&\prob \left(\mathcal{L} > \Delta \sep \cdot\right) = \prob \left(\mathcal{L} > \Delta \sep \bar{\mathcal{L}} \geq \mathcal{L}, \cdot \right) \prob \left(\bar{\mathcal{L}} \geq \mathcal{L}\sep \cdot \right) + \prob \left(\mathcal{L} > \Delta \sep \bar{\mathcal{L}} < \mathcal{L} , \cdot\right)\left(1-\prob (\bar{\mathcal{L}} \geq \mathcal{L} ) \sep \cdot \right) \leq  \nonumber\\
	&\prob \left(\mathcal{L} > \Delta \sep \bar{\mathcal{L}} \geq \mathcal{L}, \cdot \right) + 1-(1-\alpha)^2 \leq c +2\alpha - \alpha^2. 
	\end{align}
	but $c = \frac{\prob \left(\bar{\mathcal{L}} > \Delta \sep b_k, \pi, \pi', \nu \right)}{(1-\alpha)^2}$. 
	We showed that 
	\begin{align}
	&\prob \left(\mathcal{L} > \Delta \sep \cdot \right) \leq \frac{\prob \left(\bar{\mathcal{L}} > \Delta \sep b_k, \pi, \pi', \nu \right)}{(1-\alpha)^2} +2\alpha - \alpha^2. 
	\end{align}
	Furthermore, by definition of TDF
	\begin{align}
	\prob \left(\mathcal{L} > \Delta \sep \cdot \right) \leq 1.
	\end{align}
	We write the above two relations compactly as 
	\begin{align}
	\prob \left(\mathcal{L} > \Delta \sep \cdot \right) \leq \beta(\Delta),
	\end{align}
	where $\beta(\Delta) = \min \left\{1, \frac{\prob \left(\bar{\mathcal{L}} > \Delta \sep b_k, \pi, \pi', \nu \right)}{(1-\alpha)^2} +2\alpha - \alpha^2 \right\}$. Clearly
	\begin{align}
	&\prob \left(\mathcal{L} \leq \Delta \sep b_k, \pi, \pi', \nu\right)  = 1- \prob \left(\mathcal{L} > \Delta \sep b_k, \pi, \pi', \nu\right)  \geq  1 -   \beta(\Delta).
	\end{align}
	This completes the proof.
\end{proof}

\section{Calculation of \ploss Offline} \label{app:Calc_ploss_offline}
Similar to \pbloss, one approach to obtain \ploss \emph{offline} is to sample $(\ret,  \ret', \retsimpl, \retsimpl')\sim \probd (\ret, \ret', \retsimpl, \retsimpl'  \sep  b_k, \pi, \pi', \nu)$.  \ploss is represented by $\{ f_{\mathcal{L}}(\ret, \ret', \retsimpl, \retsimpl')\}$. 

To generate samples  $(\ret,  \ret', \retsimpl, \retsimpl')$ we marginalize over future measurements $z_{k+}\equiv z_{k+1:k+L}$  and $z'_{k+}\equiv z'_{k+1:k+L}$.  As we mentioned in the main manuscript 
\begin{align}
\int_{{\substack{z_{k+} \\ z'_{k+} }}}  \probd (\ret, \ret', \retsimpl, \retsimpl'  \sep  b_k, \pi, \pi', \nu, z_{k+}, z'_{k+} ) \probd (z_{k+}, z'_{k+} \sep b_k, \pi, \pi') dz_{k+} dz'_{k+} =
\nonumber \\
\int_{{\substack{z_{k+} \\ z'_{k+} }}}  \probd (\ret, \retsimpl \sep  \his_{k+L}, \nu) \probd (\ret', \retsimpl'  \sep  \his'_{k+L}, \nu) \cdot \probd (z_{k+}, z'_{k+} \sep b_k, \pi, \pi') dz_{k+} dz'_{k+}.
\end{align}
We take samples of $ \probd (z_{k+}, z'_{k+} | b_k, \pi, \pi')$  from the corresponding extended belief trees built for \pbloss. To sample 
\begin{align}
\probd (\ret, \retsimpl \sep  \his_{k+L}, \nu) \quad 	, \quad \probd (\ret', \retsimpl' \sep  \his'_{k+L}, \nu), 
\end{align} 
we use the original (not simplified) rewards calculated from the beliefs present at the belief tree (belief tree does not undergo simplification) and their simplified counterparts. 

\section{Derivation of the Distribution of a Linear Function of Gaussian Random Vector} \label{app:BoundsDer}
We are interested in the following distribution 
\begin{align}
A \cdot \begin{pmatrix} \ret \\  \retsimpl \end{pmatrix},
\end{align} 
where $A= \begin{pmatrix} 1 & -1  \end{pmatrix}$. Denote by $\phi_x(t)=\exp \left(it^T\mu-\frac{1}{2}t^T \Sigma t\right)$ the characteristic function of  $x=\begin{pmatrix} \ret \\  \retsimpl \end{pmatrix}$ and by $\phi_y(t)$ the desired characteristic function (of $y=Ax$), we have that
\begin{align}
\phi_y(t) =& \mathbb{E}\left[\exp\left(it^TAx\right)\right]=\\
&\mathbb{E}\left[\exp\left(i(A^Tt)^Tx\right)\right]=\\
&\phi_x(A^Tt)=\exp \left(i(A^Tt)^T\mu-\frac{1}{2}(A^Tt)^T \Sigma (A^Tt)\right)=\\
&\exp \left(it^TA\mu-\frac{1}{2}t^T A\Sigma A^T t\right).
\end{align}
In other words, $y=\ret - \retsimpl$ is zero mean  Gaussian with the following variance 
\begin{align}
\text{var}(y) = \text{se}^2(N) + \text{se}^2(n) -2\text{cov}.
\end{align}

\section{Derivation and Complexity of Two Dependent Extended Belief Policy Trees} \label{app:TreesCompl}
In a theoretical form, 
\begin{align}
&\probd \left(z_{k+1:k+L}, z'_{k+1:k+L} \sep b_k, \pi, \pi' \right)  = \probd \left(z_{k+1}, z'_{k+1} | b_k, \pi_k, \pi'_k \right)\nonumber \\
& \prod_{\ell=k+2}^{k+L-1} \probd \left(z_{\ell}  \sep b_{\ell-1}, \pi_{\ell-1} \right) \probd \left( z'_{\ell} \sep b'_{\ell-1}, \pi'_{\ell-1} \right) \int_{b_{\ell}} \probd \left( b_{\ell} \sep b_{\ell-1}, \pi_{\ell-1}, z_{\ell} \right)    \int_{b'_{\ell}} \probd \left( b'_{\ell} \sep b'_{\ell-1}, \pi'_{\ell-1}, z'_{\ell} \right)  \nonumber \\
& \probd \left(z_{k+L} \sep b_{k+L-1}, \pi_{k+L-1}\right) \probd \left(z'_{k+L} \sep b'_{k+L-1}, \pi'_{k+L-1}\right). \label{eq:MutialLikelihood}
\end{align}	
However, realizations of the future are correlated through the belief from planning time (present),
\begin{align}
&\probd \left( z_{k+1}, z'_{k+1} \sep b_k, \pi_k, \pi'_k \right) = \int_{x_{k+1}, x'_{k+1} } \observ{x_{k+1}}{z_{k+1}} \observ{x'_{k+1}}{z'_{k+1}} \int_{x_k} \probd \left(x_{k+1}, x'_{k+1}| x_k, b_k, \pi_k, \pi_k' \right)b_k(x_k)= \nonumber\\
&\int_{x_{k+1}, x'_{k+1} } \observ{x_{k+1}}{z_{k+1}} \observ{x'_{k+1}}{z'_{k+1}} \int_{x_k} \mot{x_k}{\pi_k(b_k)}{x_{k+1}} \mot{x_k}{\pi'_k(b_k)}{x'_{k+1}} b_k(x_k). \label{eq: observations_creation}
\end{align} 
In practice, the mutual likelihood of the observations $ \probd (z_{k+}, z'_{k+} | b_k, \pi, \pi')$ \eqref{eq:MutialLikelihood}  corresponds to two extended belief policy trees, starting from the same root ($b_k$) and having the same rule for choosing rollouts.

Below we discuss the construction of the extended belief tree. 
Let $N$ be a number of samples of the posterior belief. 
From now let us assume that  $\psi_{st}$ is an off-the-shelf particle filter with low-variance re-sampling \cite{thrun2005probabilistic}. The entire belief update process complexity is $\mathcal{O}(N)$. We choose the samples of the belief for creating the observations according to the  following scheme, which is inspired by  progressive widening \cite{sunberg2017online}. Let $n^{(\ell)}_z$ be number of observations generated by each belief at level $\ell$ of the tree. We specify $n_z^{(1)}$ (the number of observations generated by $b_k$) and the dwindle factor $c$. Starting from $\ell=2$ the number of observations generated by each belief on  level $\ell$ in the tree is calculated as $n^{(\ell)}_z=\max\{1, \lfloor \frac{n^{(1)}_z}{(\ell-1)\cdot c} \rfloor \}$.
We sample observations from resampled posterior with Fisher-Yates shuffling (with early termination) \cite{knuth2014art}. This algorithm is $\mathcal{O}(N)$ for initialization, plus $\mathcal{O}(n^{(\ell)}_z)$ for random shuffling.

In our extended belief policy tree, there may be many beliefs stemming from an observation. Denote this number by $n_b$.  
The complexity of constructing the tree is  
\begin{align}
\mathcal{O}(N) \sum_{\ell=1}^{L-1} \prod_{i=1}^{\ell} n_b n^{(i)}_z. \label{eq:TreeComplexity}
\end{align}
At each level of the tree beside the bottom,  we must apply a particle filter number of times equal to the total number of the beliefs at the next level, which is $\prod_{i=1}^{\ell} n_b n^{(i)}_z$ at level $\ell$. Also, we need to subsample observations at the current level. Since the number of beliefs at the next level is not smaller than at the current level, and 
the subsampler and particle filter complexity is linear in $N$, we are left with \eqref{eq:TreeComplexity}. 
Let us mention that sampling from the belief and application of particle filters on each level can be done in parallel.

\section{Differential Entropy}
We adopt the differential entropy estimator from \cite{skoglar2009information, boers2010particle}.  Suppose  $b_k$ is given by a weighted set of samples
\begin{align}
b_k = \sum_{i=1}^n w^i_k \delta(x_k - x_k^i).
\end{align}
Propagated belief is 
\begin{align}
b^{-}_{k+1} = \sum_{i=1}^n w^i_k \mot{x_k^i}{a_k}{x_{k+1}} \approx \sum_{i=1}^n w^i_k \delta(x_{k+1}-x^i_{k+1 | k}). \label{eq:ParticlesPropagated}
\end{align} 
The posterior is 
\begin{align}
b_{k+1}   \approx \sum_{i=1}^n w^i_{k+1} \delta(x_{k+1}-x^i_{k+1 | k}).
\end{align} 
We plug the particle approximation of the propagated belief \eqref{eq:ParticlesPropagated} into the likelihood of the observation in subsequent time instant and obtain
\begin{align}
\probd \left(z_{k+1} \mid b_k , a_k \right)= \int_{x_{k+1}} \observ{x_{k+1}}{z_{k+1}}b^{-}_{k+1} (x_{k+1}) \approx \sum_{i=1}^n w^i_k \observ{x^i_{k+1 | k}}{z_{k+1}}.
\end{align}
We assumed here that weights are normalized. Using Bayes rule,  we have
\begin{align}
\mathcal{H}(b_{k+1}) = -& \int_{x_{k+1}} b_{k+1}(x_{k+1}) \ln \observ{x_{k+1}}{z_{k+1}} \nonumber \\
- & \int_{x_{k+1}} b_{k+1}(x_{k+1}) \ln b^{-}_{k+1} (x_{k+1}) + \ln \probd \left(z_{k+1} \mid b_k , a_k \right),
\end{align}
and, therefore,
\begin{align}
\mathcal{H}(b_{k+1}) = -& \sum_{j=1}^n w^j_{k+1} \ln \observ{x^j_{k+1 | k}}{z_{k+1}} \nonumber \\
- & \sum_{j=1}^n w^j_{k+1}\ln\sum_{i=1}^n w^i_k \mot{x_k^i}{a_k}{x^j_{k+1 | k}} + \ln \sum_{i=1}^n w^i_k \observ{x^i_{k+1 | k}}{z_{k+1}}. \label{eq: differential_entropy}
\end{align}

\section{Additional Results and Discussions} \label{app:addRes}
This section shows additional results for two of the scenarios we presented in the main manuscript. In all our simulations we set $w=0.1$, $m=50$, $\alpha = 0.01$, $z_{\alpha/2}=2.56$, and the total number of observations is $500$.     
\paragraph{Characterizing Probabilistic Action Consistency}
In this scenario
\begin{align}
\Sigma_v(x)= v(x) \cdot I, \quad	v(x) = w \cdot \min \{1, \|x - x^*\|^2_2\},
\end{align}
where $x^*$ is the location of the light beacon closest to $x$. 
\begin{figure}[t] 
	\centering
	{\includegraphics[width=\textwidth]{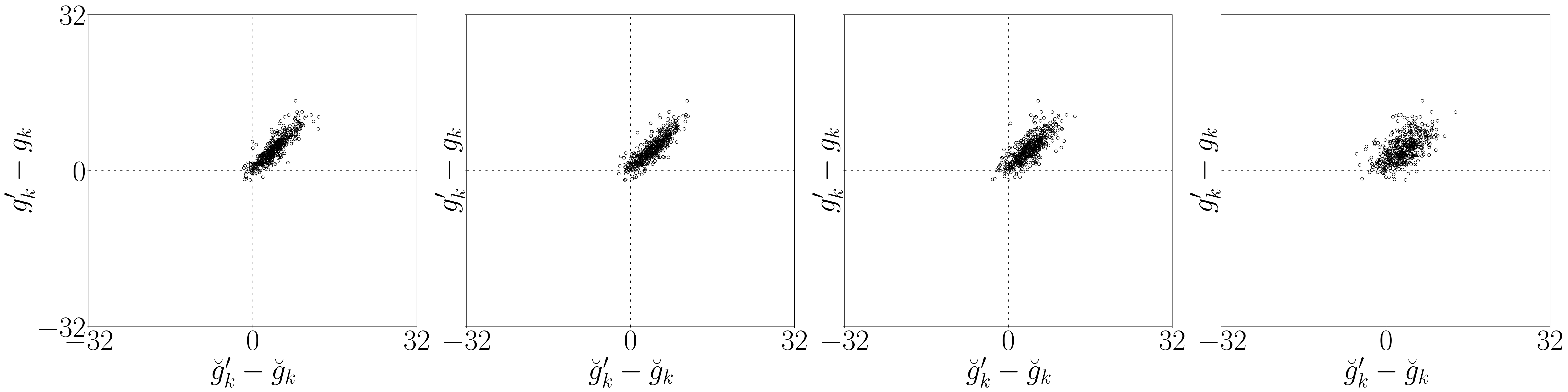}}
	\caption{Results for scenario $1$ - probabilistic action consistency: We demonstrate from the left to the right action consistency of the samples of the returns for  $n=175, n=125, n=75, n=25$, whereas $N=1500$. As we see, samples violating action consistency are present at all graphs.} \label{fig:DiffsSimScen}
\end{figure}

\begin{table}[t]
	\centering
	\begin{tabular}{|c|c|c|c|c|c|c|c|}
		\hline 
		& $n=175$ & $n=150$ & $n=125$ & $n=100$ & $n=75$ & $n=50$ & $n=25$ \\
		\hline 
		$ g_k $ time [sec] & 104957 & 69658 & 95651 & 69713 & 68584 & 96354 & 66513 \\
		\hline 
		$ \breve{g}_k $ and $ l, u  $ time [sec] & 72694 & 34842 & 33759 & 15498 & 8293 & 5589 & 969 \\
		\hline 
		$ \breve{g}_k $ time [sec] & 1454 & 661 & 669 & 298 & 172 & 119 & 14 \\
		\hline 
		$ l, u $ time [sec] & 71240 & 34181 & 33090 & 15200 & 8121 & 5469 & 955 \\
		\hline 
	\end{tabular}
	
	\caption{Results for scenario $1$ - probabilistic action consistency: run times for $N=1500$.}
	\label{tbl:run_times}
\end{table}

\begin{figure}[t]
	\centering
	{\includegraphics[width=\textwidth]{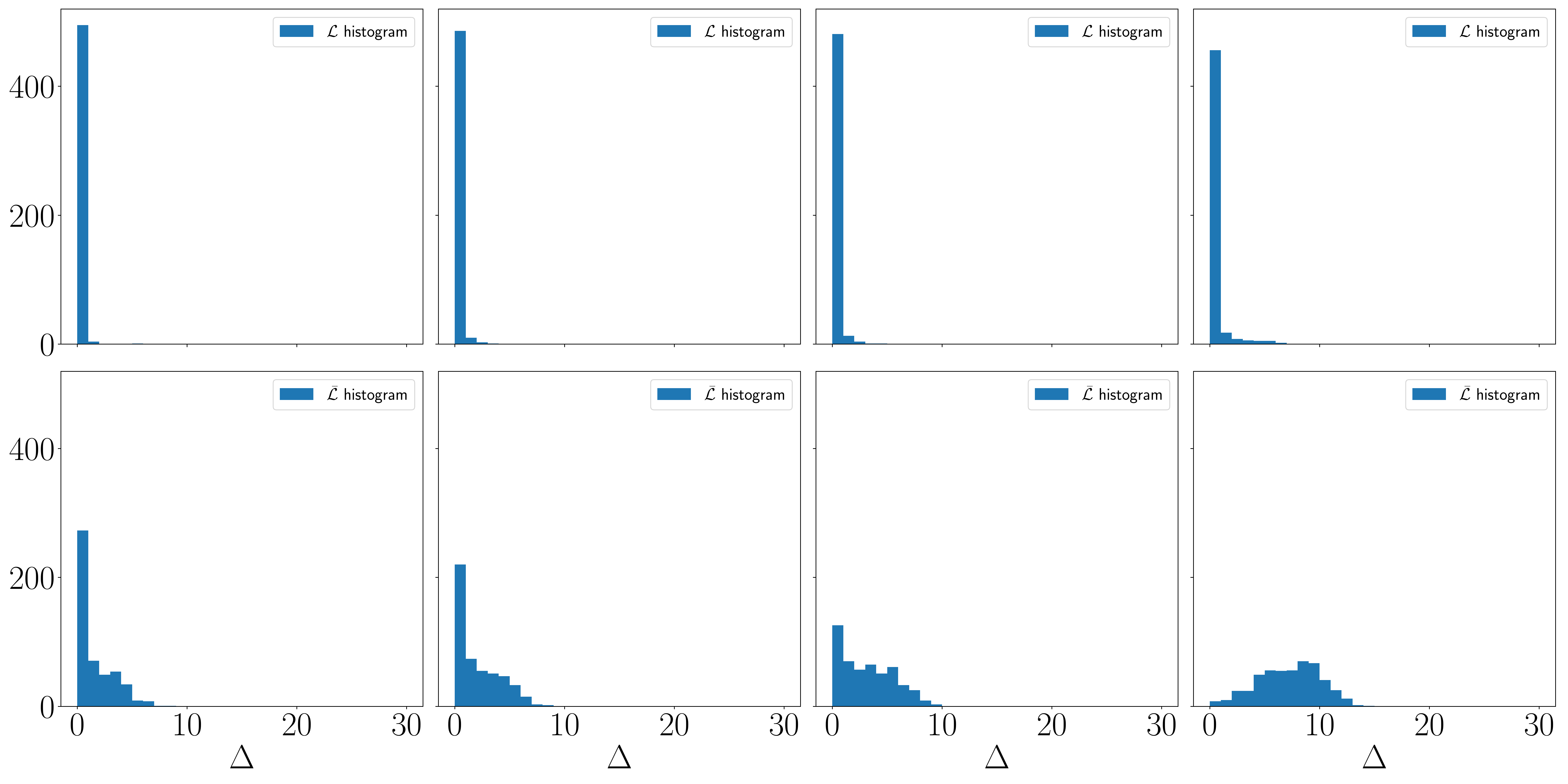}}
	\caption{Results for scenario $1$ - probabilistic action consistency: Histograms of \ploss and \pbloss for  $N=1500$, $\alpha = 0.01$, $z_{\alpha/2}=2.56$, bin width is $1.0$;  from the left to the right $n=175$, $n=125$, $n=75$, $n=25$.  } 
	\label{fig:HistsScenario1}
\end{figure}

\begin{figure}[t]
	\centering
	{\includegraphics[width=\textwidth]{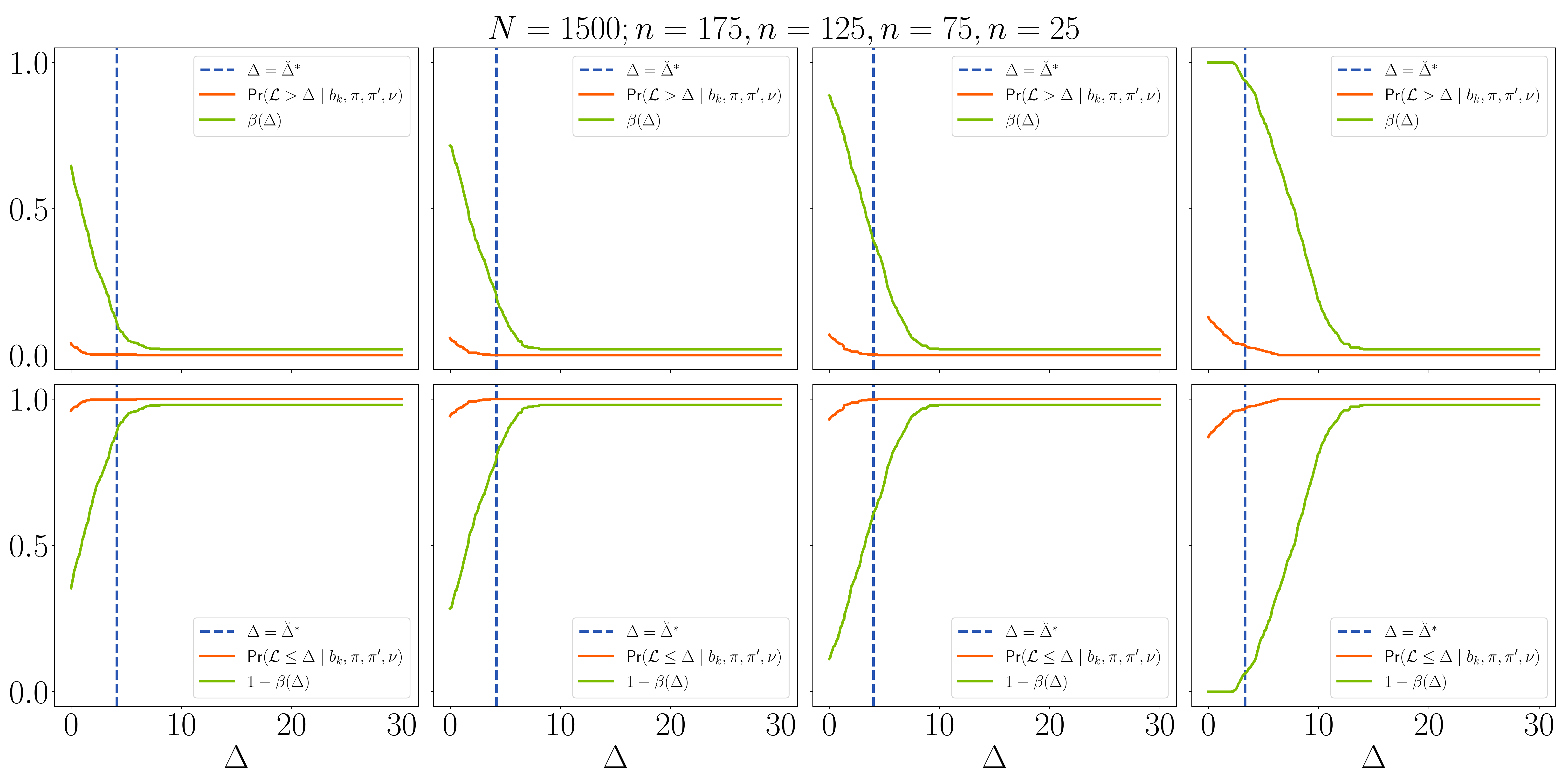}}
	\caption{Results for scenario $1$ - probabilistic action consistency: Empirical characterization for $N=1500$, $\alpha = 0.01$, $z_{\alpha/2}=2.56$, evaluated in a grid with intervals  $0.001$; from the left to the right $n=175$, $n=125$, $n=75$, $n=25$.} \label{fig:Scenario1Study}
\end{figure}

We showed an illustration of this scenario in Fig.~\ref{fig:SimScen}. In Fig.~\ref{fig:DiffsSimScen}, we demonstrated scatter plots that show samples of the
simplified and original returns' differences. We identify that with decreasing $n$, more samples are not action consistent. This phenomenon is corroborated by the histograms of $\mathcal{L}$ in Fig.~ \ref{fig:HistsScenario1}. 
We recite the model  
\begin{align}
&\begin{pmatrix} \ret \\  \retsimpl \end{pmatrix} \sep \his_{k+L}, \nu \sim \mathcal{N} 
\left(\begin{pmatrix} \mu \\ \mu \end{pmatrix}, 
\begin{pmatrix}
\text{se}^2(N) & \text{cov} \\ \text{cov} & \text{se}^2(n) 
\end{pmatrix} \right). \label{eq:Modelapp}
\end{align}
With decreasing $n$, the variance $(\text{se}^2(n))$ in \eqref{eq:Modelapp} grows, making more samples of the returns are not action consistent. Moreover, since our probabilistic bounds on the return are based on $(\text{se}(n))$, it is harder to characterize \ploss with \pbloss. This is observable in histograms of  $\bar{\mathcal{L}}$ in Fig.~\ref{fig:HistsScenario1} and empirical characterization in Fig.~\ref{fig:Scenario1Study}, where 
\begin{align}
&\text{empirical \ploss TDF:}  \quad \frac{\text{number of samples of }\mathcal{L} \text{, satisfying } \mathcal{L} > \Delta }{\text{number of all samples of } \mathcal{L}}, \\
&\text{empirical \pbloss TDF:}  \quad \frac{\text{number of samples of }\bar{\mathcal{L}} \text{, satisfying } \bar{\mathcal{L}} > \Delta }{\text{number of all samples of } \bar{\mathcal{L}}}.
\end{align}

\paragraph{Revealing Empirical Absolute Action Consistency} 
Here the covariance matrix of the observation model is   
\begin{align}
\Sigma_v(x)= v(x) \cdot I, \quad v(x) = w \cdot  \|x - x^*\|^2_2.
\end{align}
\begin{figure}[t] 
	\centering
	{\includegraphics[width=\textwidth]{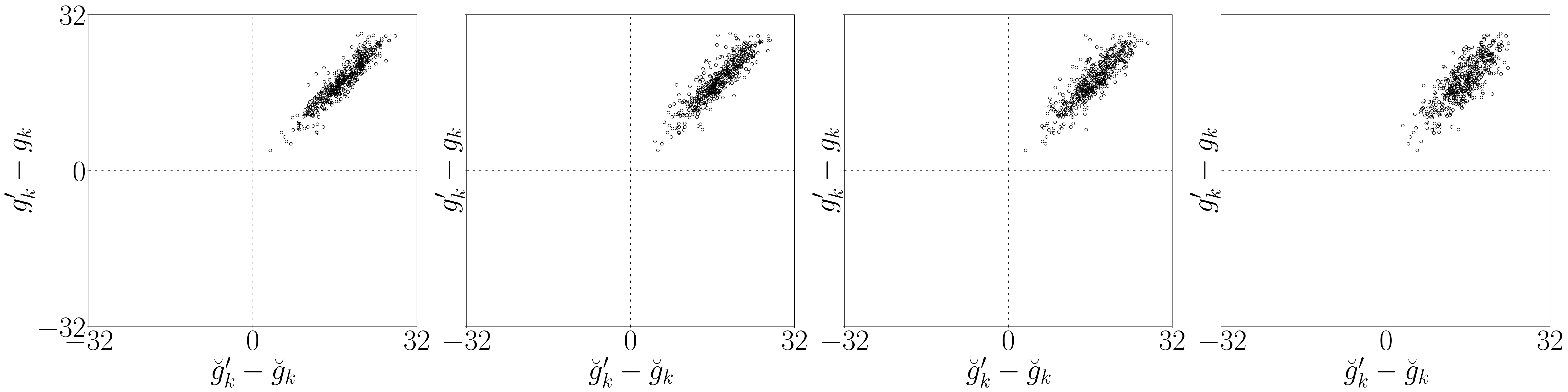}}
	\caption{Results for scenario $2$ - empirical absolute action consistency: We demonstrate from the left to the right action consistency of the samples of the returns for  $n=100, n=75, n=50, n=25$, whereas $N=1000$. As we see, all the samples are action consistent.} \label{fig:DiffsScenario2}
\end{figure}
\begin{figure}[t]
	\centering
	{\includegraphics[width=\textwidth]{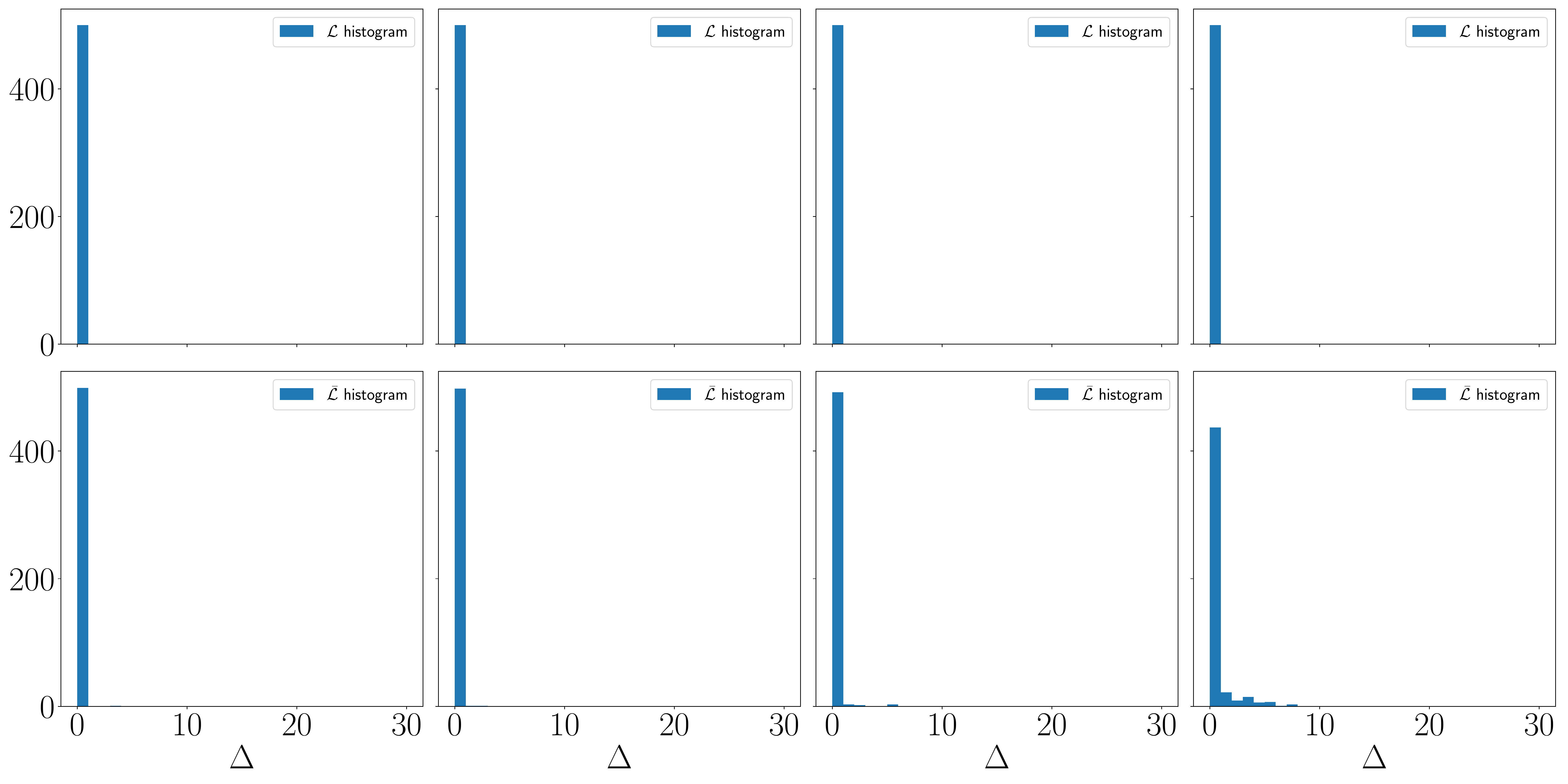}}
	\caption{Results for scenario $2$ - empirical absolute action consistency: Histograms of \ploss and \pbloss for $N=1000$, $\alpha = 0.01$, $z_{\alpha/2}=2.56$, bin width is $1.0$;  from the left to the right $n=100$, $n=75$, $n=50$, $n=25$.} 
	\label{fig:HistsScenario2}
\end{figure}
\begin{figure}[t]
	\centering
	{\includegraphics[width=\textwidth]{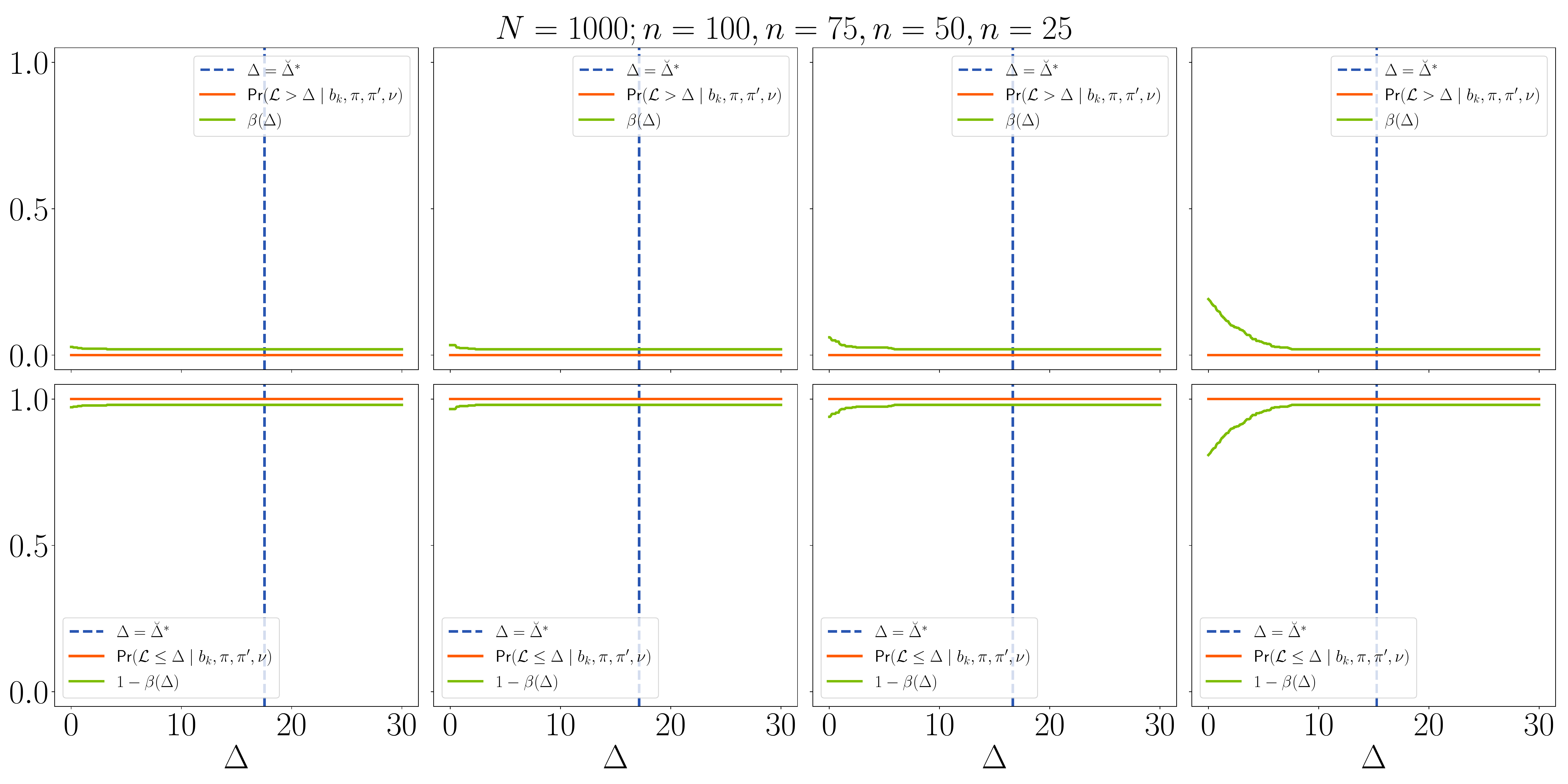}}
	\caption{Results for scenario $2$ - empirical absolute action consistency: Empirical characterization for $N=1000$, $\alpha = 0.01$, $z_{\alpha/2}=2.56$, evaluated in a grid with intervals  $0.001$; from the left to the right $n=100$, $n=75$, $n=50$, $n=25$.} \label{fig:Scenario2Study} 
\end{figure}
\begin{table}[t]
	\centering
	\begin{tabular}{|c|c|c|c|c|}
		\hline 
		& $n=100$ & $n=75$ & $n=50$ & $n=25$ \\
		\hline 
		$ g_k $ time [sec] & 36745 & 45187 & 44899 & 30889 \\
		\hline 
		$ \breve{g}_k $ and $ l, u  $ time [sec] & 17361 & 12546 & 4388 & 844 \\
		\hline 
		$ \breve{g}_k $ time [sec] & 363 & 247 & 65 & 14 \\
		\hline 
		$ l, u $ time [sec] & 16998 & 12299 & 4323 & 830 \\
		\hline 
	\end{tabular}
	\caption{Results for scenario $2$ - empirical absolute action consistency: run times for $N=1000$.}
	\label{tbl:run_times2}
\end{table}
We demonstrated this scenario in Fig.~\ref{fig:SimScen2}. As we can see in Fig.~\ref{fig:DiffsScenario2}, the clouds of samples are farther from the origin than in the previous scenario. Therefore, two action sequences are more distant. In this case, the simplification is empirically absolute action consistent, as we observe in the histograms of  $\mathcal{L}$ in Fig.~\ref{fig:HistsScenario2} and empirical characterization shown in Fig.~\ref{fig:Scenario2Study}.  We report run times for two scenarios in Table~\ref{tbl:run_times} and Table~\ref{tbl:run_times2}, respectively.

\section{Technical Characteristics of Computers Used in Simulations}
Our simulations are written in Julia language with a multi-threaded calculation of immediate reward. We used $4$ computers with the following characteristics:

\begin{enumerate}
	\item 40 cores Intel(R) Xeon(R)  E5-2670 v2 with 256 GB of RAM  working at 2.50GHz; 
	\item 24 cores Intel(R) Core(TM) i9-7920X with 64 GB of RAM  working at 2.90GHz;	
	\item 20 cores Intel(R) Xeon(R)  E5-2630 v4 with 64 GB of RAM  working at 2.20GHz;	  
	\item 20 cores Intel(R) Core(TM) i9-9820X with 64 GB of RAM working at 3.30GHz.
	
\end{enumerate}	
\end{appendices}
\end{document}